%% file: arxiv.tex
\documentclass{article}
\usepackage{arxiv}
\usepackage{times}

\input{math_commands.tex}

\usepackage[colorlinks,citecolor=blue,linkcolor=red,urlcolor=blue]{hyperref}
\usepackage{url}
\usepackage{tikz}
\usetikzlibrary{arrows.meta,positioning,calc,fit,backgrounds}
\usepackage{amsthm, mathtools}
\usepackage{cleveref}
\usepackage{booktabs}
\usepackage{array}
\usepackage{enumitem}

\usepackage[
  style=alphabetic,
  natbib=true,
  backend=biber,
  maxcitenames=2,   % "Mandal et al. [MTR23]" in the text
  maxbibnames=99,   % full author lists in the reference list
  maxalphanames=4,  % label conventions of classic alpha.bst:
  minalphanames=3,  %   <=4 authors -> one initial each; more -> 3 + "+"
  giveninits=false,
  doi=true, url=true, isbn=false, eprint=true
]{biblatex}
\DeclareFieldFormat[article,inbook,incollection,inproceedings,thesis,unpublished]%
  {title}{#1}

\renewbibmacro{in:}{\printtext{In\space}}

\DeclareLabelalphaTemplate{
  \labelelement{
    \field[strwidth=3,strside=left,ifnames=1]{labelname}
    \field[strwidth=1,strside=left]{labelname}
  }
  \labelelement{\field[strwidth=2,strside=right]{year}}
}

\newtheorem{theorem}{Theorem}
\newtheorem{proposition}{Proposition}
\newtheorem{lemma}{Lemma}
\newtheorem{corollary}{Corollary}
\theoremstyle{definition}
\newtheorem{definition}{Definition}
\newtheorem{assumption}{Assumption}

\newcommand{\Gloc}{\mathrm{Gap}^{\mathrm{loc}}}
\newcommand{\Gglob}{\mathrm{Gap}^{\mathrm{glob}}}

\newcommand{\set}[1]{\left\{#1\right\}}
\newcommand{\abs}[1]{\left|#1\right|}
\newcommand{\bpi}{\boldsymbol{\pi}}
\newcommand{\Glocm}[1]{\mathrm{Gap}^{\mathrm{loc}}_{#1}}
\newcommand{\Gglobm}[1]{\mathrm{Gap}^{\mathrm{glob}}_{#1}}
\usepackage{algorithm}
\usepackage{algpseudocode}

\usepackage[toc,page,header]{appendix}
\usepackage{etoc}

\newcommand{\norm}[1]{\left\lVert #1 \right \rVert}

\title{Local and Global Stability in\\Performative Reinforcement Learning}

\renewcommand{\undertitle}{}

\renewcommand{\shorttitle}{Local and Global Stability in Performative RL}

\renewcommand{\headeright}{}

\author{%
  Debmalya Mandal \\
  Department of Computer Science \\
  University of Warwick, UK \\
  \texttt{Debmalya.Mandal@warwick.ac.uk} \\
}

\hypersetup{
  pdftitle={Local and Global Stability in Performative Reinforcement Learning},
  pdfauthor={Debmalya Mandal},
  pdfkeywords={performative reinforcement learning, performative stability,
               online learning, Markov games}
}

\begin{document}

\maketitle

\begin{abstract}
In performative reinforcement learning the deployed policy shapes the environment that
generates the learner's future data, and the natural solution concept is a
\emph{performatively stable} policy that is optimal in the environment it induces. Existing
convergence guarantees rely on Lipschitz sensitivity assumptions on the environment map
$\pi \mapsto (P_\pi, r_\pi)$, which are hard to verify and fail in settings such as
multi-agent best-response dynamics. We instead study stability for \emph{mixtures} of
policies, and show that the resulting picture is fundamentally different from performative
prediction, where randomization removes the need for any sensitivity assumption. We
distinguish \emph{local} mixed stability, an occupancy-weighted first-order relaxation that we
show is equivalent to stationarity, from \emph{global} mixed stability, which certifies against
arbitrary deviating policies. Our first result is that a weighted per-state Hedge dynamic
drives the local stability gap to zero at an $O(1/\sqrt{T})$ rate for an \emph{arbitrary},
possibly discontinuous, environment map, both with exact and with trajectory feedback. The two
notions genuinely differ: we exhibit an instance where local stability is achieved exactly but
every mixture has global stability gap bounded away from zero. For global stability we
introduce a bounded transition range assumption, strictly weaker than Lipschitz sensitivity,
under which unweighted per-state Hedge converges up to a floor of
$O(\gamma\eps_P/(1-\gamma)^3)$, and we prove a matching-in-$\eps_P$ lower bound of
$\Omega(\gamma\eps_P/(1-\gamma))$ under trajectory feedback, so this floor is unavoidable.
Finally, we extend both notions to $n$-player performative Markov games, obtaining local
stability with no assumption on the joint environment map or game structure, and global
stability for performative Markov potential games.
\end{abstract}

\section{Introduction}

%\subsection{Motivation}

The classic framework of reinforcement learning~\citep{SB98} assumes that the environment is \emph{static}, i.e. the reward and the transition probabilities do not change while the policy is being learned. However, in many reinforcement learning applications the deployed policy shapes the very environment
that generates the learner's future data. For example, a recommendation policy can change population behaviour through interactions~\citep{JCLG+19}, and
traffic-routing policy changes congestion patterns~\citep{WKPV+21}. Moreover, in multi-agent settings~\citep{ZYB21}, as the learner adapts its policy, the opponents and/or followers adjust their policies, often leading to arbitrary shift in the underlying environment.  \citet{MTR23} formalised this
phenomenon as \emph{performative reinforcement learning} and introduced \emph{performative
stability} as the natural fixed-point solution concept — a policy that remains optimal in the
MDP it induces. Their result generalizes the framework of \emph{performative prediction} from supervised learning~\citep{PZMH20}, and shows that under sensitivity assumption on both the \emph{reward} and \emph{transition kernel} (i.e. the map $\pi \mapsto (P_\pi, r_\pi)$), a regularized variation of repeated optimization converges to a performatively stable policy. 

%however, rely on a Lipschitz-type sensitivity assumption on the
%environment map $\pi \mapsto M(\pi) = (P_\pi, r_\pi)$, and %search only over \emph{pure} policies,
%even though, as in performative prediction \citep{FP26}, an exactly stable pure policy need not
%exist while a stable \emph{mixture} of policies always does.

Although repeated optimization can converge to a performatively stable policy, it requires strong sensitivity assumptions on the reward and transition kernel~\citep{MTR23, MR25}. These assumptions are hard to verify, and may not hold in certain applications. For example, in multi-agent settings, best response dynamics are often discontinuous and chaotic~\citep{CP19}, and if the learner changes its policy, the opponents may change their policies adversarially~\citep{MNMS+23}. In the absence of sensitivity assumptions, a \emph{deterministic} performatively stable policy may not exist, and a natural question what should be the appropriate solution concept in this case. Fortunately, recent work in performative prediction~\citep{FP26} shows that if the learner is allowed to randomize over a mixture of policies, then performative stability can be achieved in expectation with no sensitivity assumption on the environment map. The main result of \citet{FP26} is that the iterates of an online no-regret algorithms naturally lead to such a stable mixture of policies.

This motivates us to study stability in performative reinforcement learning through a mixture of policies in the absence of sensitivity assumptions. However, we discover that the notion of stability in performative reinforcement learning is fundamentally different than in supervised learning~\citep{PZMH20}. In reinforcement learning, the environment consists of two components: a \emph{reward} map ($\pi \mapsto r_\pi$) and a \emph{transition} map ($\pi \mapsto P_\pi$). When both maps are performative, even a mixture of policies can fail to achieve the standard notion of \emph{global stability}~\citep{FP26}. The primary reason is that achieving sublinear regret inonline reinforcement learning with \emph{changing transition kernel} is both computationally hard~\citep{ABKSS13} and information-theoretically hard with bandit feedback~\citep{TWYS21}. Therefore, we introduce the notion of \emph{local stability}, a first-order relaxation of global stability that is also equivalent to the notion of \emph{stationarity}~\citep{LW24}, and study two questions:
\begin{enumerate}[label=(\roman*)]
  \item How much of the sensitivity
assumption in \citet{MTR23} is actually needed to drive an online learning algorithm toward a performatively
 stable mixture?
\item Does the answer depend on which notion of stability we
target: \emph{local} stability against first-order deviations, or \emph{global} stability against arbitrary deviations? 
\end{enumerate}
We find that the mixture of policies generated by an online algorithm can converge to performative local stability \emph{unconditionally}. However, in a stark contrast to performative prediction~\citep{FP26}, unconditional convergence to global stability is not possible even with mixtures of policies. We then show that when the \emph{bounded transition range}, an assumption weaker than sensitivity of transition map, holds, mixture of policies generated by an online algorithm can converge to performative global stability. Finally, we extend both notions of stability to $n$-player performative Markov games, and show that local stability can be achieved unconditionally, while global stability can be achieved in performative Markov potential games. Throughout we make no assumption about the sensitivity of the reward map.

\subsection{Overview of the main contributions}

%We study two solution concepts for mixtures of policies: \emph{global} mixed performative
%stability (\Cref{def:glob}), which asks that no single deviating policy — evaluated in the
%environment it would induce — improve on the mixture by more than $\eps$, and \emph{local} mixed
%performative stability (\Cref{def:loc}), an occupancy-weighted first-order relaxation that we show
%coincides with stationarity (\Cref{lem:equivalence-local-stationary}) and is a strictly weaker
%requirement (\Cref{lem:zero-local,lem:positive-global}). Our contributions are as follows.
We study local and global mixed performative stability, and our contributions are as follows.
\begin{enumerate}[leftmargin=2em, label=(\alph*)]
  \item \textbf{Local stability needs no assumption.} We first propose a weighted per-state Hedge algorithm
 that achieves arbitrarylocal stability at an $O(1/\sqrt T)$ rate without any assumption on the environment map. We then consider a setting where the learner only accesses the induced environment through sampled trajectories and give a
  finite-sample version of our algorithm with sample complexity $\tilde{O}\left(\tfrac{SA}{(1-\gamma)^4 \epsilon^2}\right)$ to achieve $\epsilon$-local stability.

  \item \textbf{A strict separation between local and global stability.} We exhibit a \emph{survival-chain
  instance} on which local stability is exactly achievable, yet every
mixture of policieshas global stability gap bounded away from $0$. This establishes a strict separation between local and global notion of stability in performative reinforcement learning, in constrast to recent results in performative prediction~\citep{FP26}. 

  \item \textbf{Global stability under bounded transition range.} We introduce the assumption of \emph{Bounded Transition Range}, an assumption strictly weaker than the  sensitivity of transition map ($\pi \mapsto P_\pi$)~\citep{MTR23}. Under this assumption, we propose an unweighted per-state Hedge algorithm  that attains global stability
  up to a floor of $O(\gamma \epsilon_P/(1-\gamma)^3)$, alongwith a
  finite-sample counterpart with similar sample complexity as part (a). We show this floor is not an artefact of
  our analysis, and under a local trajectory-feedback protocol, no algorithm can attain global stability gap better than 
$\Omega(\gamma \epsilon_P/(1-\gamma))$.

  \item \textbf{Extension to performative Markov games.} We then generalise both notions of stability
  to $n$-player performative Markov games. An independent
  weighted per-state Hedge algorithm attains local stability with \emph{no}
  assumption on the joint sensitivity of the environment map or on the game's structure, improving on multiplayer performative prediction
  results that require such conditions \citep{NFDFR23}. It turns out that the $n$-player Markov games is more challenging, and a two-player instance can separate global and local stability even without any performative effects. Finally, we show that we can recover a global stability guarantee for
  performative Markov \emph{potential} games using independent
  projected policy gradient ascent by the agents.
\end{enumerate}

\Cref{tab:contributions} summarises the assumptions, convergence rate, and convergence floor of
our four main guarantees.

\begin{table}[t]
\centering
\footnotesize
\renewcommand{\arraystretch}{1.6}
\setlength{\tabcolsep}{3.5pt}
\begin{tabular}{@{}>{\centering\arraybackslash}p{0.85in}>{\centering\arraybackslash}p{1.05in}>{\centering\arraybackslash}p{2.0in}>{\centering\arraybackslash}p{1in}@{}}
\toprule
\textbf{Setting} & \textbf{Assumptions} & \textbf{Speed of convergence} & \textbf{Convergence floor} \\
\midrule
Local stability
& None 
& $O\!\left(\tfrac{1}{(1-\gamma)^2}\sqrt{\tfrac{S\ln A}{T}}\right)$ (\Cref{thm:A})
& $0$ \\
\midrule
Global stability
& Bounded transition range (Assn. \ref{ass:range})
& $O\!\left(\tfrac{1}{(1-\gamma)^2}\sqrt{\tfrac{\ln (SA)}{T}}\right)$ (\Cref{thm:global-stability})
& $\Theta\!\left(\tfrac{\gamma\,\eps_P}{(1-\gamma)^3}\right)$\\
\midrule
Multi-agent local stability
& None 
& $O\left(\tfrac{1}{(1-\gamma)^2}\sqrt{\tfrac{S\ln A_{\max}}{T}}\right)$ (\Cref{thm:multi-agent-local-stability})
& $0$ \\
\midrule
Multi-agent global stability
& Markov potential game (Assn. \ref{ass:mpg},~\ref{ass:mpgreg})
& $O\!\left(\kappa\sqrt{\tfrac{Sn}{T(1-\gamma)}}\right)$ (\Cref{thm:convergence-potential-mpg})
& $\Theta\!\left(\kappa\sqrt{Sn}\,\eps_L\right)$ \\
\bottomrule
\end{tabular}
\caption{Summary of our stability guarantees. Here $\eps_P$ is the bounded-transition-range
parameter of \Cref{ass:range}, $\eps_L$ is the potential-sensitivity parameter of
\Cref{ass:mpgreg}, $A_{\max} := \max_i A_i$, and
$\kappa$ is the distribution mismatch coefficient of \Cref{thm:convergence-potential-mpg}.
``Convergence floor'' is the residual stability gap that persists as $T \to \infty$. ``None''
means the guarantee holds for an arbitrary, possibly discontinuous, environment map.}
\label{tab:contributions}
\end{table}

\subsection{Related work}

\textbf{Performative prediction and its stability.} \citet{PZMH20}
first formalised the general phenomenon as \emph{performative prediction}: the deployed model shifts the
data-generating distribution, and the goal is to obtain either a performatively stable or optimal solution. The main observation is that repeated retraining converges to a stable solution. Subsequently, \citet{MPZH20} and \citet{DX23} give stochastic gradient based lazy-deploymentmethods that
converge to such a stable point at a rate governed by the sensitivity of the performative map. In parallel, \citet{MPZ21} study performative risk, and obtain convergence to performatively optimal solutions under parametric assumptions on the distribution map. \citet{BHK22} and \citet{IZZY22} extend the framework to a \emph{stateful} world in which the population's response
depends on the history of deployed models rather than only on the current one. 

Our approach is closest to \citet{FP26},
who show that online learning over \emph{mixtures} of models attains performative stability at an
$O(1/\sqrt T)$ rate with no smoothness or sensitivity assumption on the performative map, whereas
earlier results require such assumptions to guarantee convergence of a deterministic policy. Our local
stability results (\Cref{sec:local-stability}) transport this mixture-based, assumption-free
approach to the reinforcement learning setting. \citet{LW24} and \citet{HT25} study stationary points of possibly non-convex
performative risk under a decision dependent, distributionally robust formulation.
 \Cref{lem:equivalence-local-stationary} shows that the stationary stability notion adapted by them is equivalent to our local stability notion.  \citet{NFDFR23} and \citet{PY23} extend
performative prediction to multiplayer games, but their convergence guarantees require a joint
sensitivity condition on the players' performative maps together with convexity-type structure on
the game. Our \Cref{thm:multi-agent-local-stability} shows that, once the target is local rather
than global stability, none of these conditions are needed in the performative Markov game
setting.

\textbf{Performative reinforcement learning.} \citet{MTR23} introduced the framework of \emph{performative reinforcement learnign}  and gave an algorithm that converges to a pure stable  policy under an
$\eps$-sensitivity (Lipschitzness) assumption on the environment map. \Cref{ass:range} is
strictly weaker than this condition as it bounds only the \emph{range}, and not the modulus of
continuity of the induced kernels. \citet{RTMR24}
 relax the \emph{instantaneous} nature of the environment map, allowing the environment to
adjust gradually across deployments, and \citet{PMR25} study robustness of performative stability
to corrupted feedback. Both directions are complementary to ours as they consider more general feedback or deployment protocol. \citet{CH25} and \citet{BDDM25}
pursue a different solution concept, giving policy-gradient-style algorithms that converge to a
\emph{performatively optimal} policy — one that is optimal in expectation over the induced
MDP. Convergence to performativel optimal policy requires regularity conditions such as regularizer dominance or softmax
parametrization as performative optimality is a stronger and specialised target than the mixed stability
notions we study in this work. However, neither of these guarantees is unconditional in the sense of
\Cref{thm:A,thm:multi-agent-local-stability}. \citet{SSYMR25} initiated the study of independent
learning in performative Markov \emph{potential} games, which is the setting of
\Cref{sec:mpg-global}. We complement their analysis with an unconditional local stability
guarantee that requires no potential-game structure at all
(\Cref{thm:multi-agent-local-stability}), and with a lower bound showing that the convergence
floor we obtain for global stability under local trajectory feedback is unavoidable.

\textbf{Policy optimisation and Markov potential games.} Our finite-sample algorithms build on
the classical policy-gradient~\citep{SMSM99,AKLM21} and mirror-descent toolkit for MDPs~\citep{SEM20, Lan23}. \citet{LOPP22} introduce
Markov potential games and show that independent policy gradient converges to a stationary
(Nash) policy in the non-performative setting; \citet{DWZJ22} sharpen the rate and identify the
distribution mismatch coefficient that also appears in \Cref{thm:convergence-potential-mpg}. Our
contribution is to adapt this line of analysis to the performative setting, where the potential
itself changes with the deployed profile. Beyond potential games, \citet{DFG20} give the first
finite-sample convergence guarantee for independent policy gradient in general two-player
zero-sum Markov games, and \citet{JLWY21} give a decentralised algorithm that converges to
coarse correlated equilibria in general-sum Markov games without any potential-game or
convexity-type assumption. The latter is structurally close to our unconditional multi-agent
local stability guarantee (\Cref{thm:multi-agent-local-stability}), though it targets the classical,
non-performative equilibrium.

\section{Setting and notation}

\textbf{MDPs.} Fix a finite state space $\gS$ with $|\gS| = S$, a finite action space
$\gA$ with $|\gA| = A$, a discount factor $\gamma \in [0,1)$ and an initial distribution
$\rho \in \Delta(\gS)$. An MDP is a tuple $M = (P, \gS,\gA,r,\rho)$ with transition kernel
$P(\cdot \mid s,a) \in \Delta(\gS)$ and reward $r : \gS \times \gA \to [0,1]$. Let
$\Pi := \Delta(\gA)^{\gS}$ denote the set of stationary (possibly stochastic) policies.

For $\pi \in \Pi$ and a kernel $P$, the normalised discounted occupancy measure is
\[
  d^{\pi}_{P}(s,a)
  \;=\; (1-\gamma)\sum_{h \ge 0} \gamma^{h}\,
        \Pr\big[s_h = s,\, a_h = a \,\big|\, s_0 \sim \rho,\, \pi,\, P\big],
\]
with state marginal $d^{\pi}_P(s) = \sum_a d^{\pi}_P(s,a)$. It is easy to verify that $\textstyle\sum_{s,a} d^{\pi}_P(s,a) = 1$. The value function of the policy $\pi$ in the MDP $M$ is defined as
$$
V_M(\pi) = \E_{\pi,P}\left[\sum_{h\ge 0} \gamma^h \cdot r(s_h,a_h) \,\big|\, s_0 \sim \rho\right] = \frac{1}{1-\gamma} \left \langle r, d^\pi_P \right \rangle. 
$$
The state-action function of policy $\pi$ at a tuple $(s,a)$ is defined as
$$
Q_M(s,a) = \E_{\pi,P}\left[\sum_{h\ge 0} \gamma^h \cdot r(s_h,a_h) \,\big|\, s_0 =s, a_0=a\right].
$$

\paragraph{Performative RL.} We recall the framework of performative reinforcement learning from \citet{MTR23}. There is no longer a single MDP, rather the MDPs are parametrized by the policies in a class $\Pi$. If a policy $\pi$ is deployed, then the underlying environment changes to an MDP $M(\pi)$ with transition kernel $P_\pi$ and reward function $r_\pi$. The learner then observes the data generated by $\pi$ in $M(\pi)$. Compared to \cite{MTR23}, we do not make any assumption about Lipschitzness of the environment map $\pi \mapsto M(\pi) = (P_\pi, r_\pi)$ unless
stated explicitly. 

Following \cite{FP26}, we work with mixtures. A mixture $\mu \in \Delta(\Pi)$ is
deployed by drawing $\pi \sim \mu$ independently at the start of each episode and running
$\pi$. The environment then responds to the \emph{realised} $\pi$. Here we assume that the environment responds to the realised policy, and not to the mixture $\mu$ as an aggregate object. If instead the environment reacts to $\mu$ the environment map should be defined on $\Delta(\Pi)$ which we do not consider here.

\textbf{Shorthand.} We will primarily consider online learning algorithm which deploy a sequence of policies. For a sequence of deployed policies $\pi_1,\dots,\pi_T$ we write
\[
  M_t := M(\pi_t) = (P_t, r_t), \qquad
  Q_t(s,a) := Q^{\pi_t}_{M_t}(s,a), \qquad
  w_t(s) := d^{\pi_t}_{P_t}(s),
\]
and for a comparator $\pi' \in \Pi$,
\begin{equation}\label{eq:Delta}
  \Delta_t(s,\pi') \;:=\; \big\langle \pi'(\cdot\mid s) - \pi_t(\cdot \mid s),\ Q_t(s,\cdot)\big\rangle .
\end{equation}
More generally, for any $\pi \in \Pi$ we write
$\Delta_\pi(s,\pi') := \langle \pi'(\cdot\mid s) - \pi(\cdot\mid s),\, Q^{\pi}_{M(\pi)}(s,\cdot)\rangle$.

% \begin{lemma}[Range]\label{lem:range}
% For all $t$, $s$ and $\pi' \in \Pi$ we have $|\Delta_t(s,\pi')| \le \tfrac{1}{1-\gamma}$.
% \end{lemma}

% \begin{proof}
% Both $\langle \pi'(\cdot\mid s), Q_t(s,\cdot)\rangle$ and
% $\langle \pi_t(\cdot\mid s), Q_t(s,\cdot)\rangle$ are convex combinations of the numbers
% $Q_t(s,a) \in [0, \tfrac{1}{1-\gamma}]$, hence both lie in $[0,\tfrac{1}{1-\gamma}]$ and
% their difference lies in $[-\tfrac1{1-\gamma}, \tfrac1{1-\gamma}]$.
% \end{proof}
The following classical result will be useful to motivate our solution concepts.
\begin{lemma}[Performance difference; \cite{KL02}]\label{lem:pdl}
For any MDP $M = (P,r)$ and any $\pi,\pi' \in \Pi$,
\[
  V_M(\pi') - V_M(\pi)
  \;=\; \frac{1}{1-\gamma}\sum_{s \in \gS} d^{\pi'}_{P}(s)\,
        \big\langle \pi'(\cdot\mid s) - \pi(\cdot\mid s),\ Q^{\pi}_M(s,\cdot)\big\rangle .
\]
\end{lemma}

\subsection{Two solution concepts}
\citet{MTR23} first introduced the notion of \emph{performative stability} in reinforcement learning. A policy $\pi$ is performatively stable if $\pi \in \argmax_{\pi'} V_{M(\pi)}(\pi')$. This means that the policy $\pi$ is optimal in the induced MDP $M(\pi)$. We now extend this definition to a mixture of policies.

\begin{definition}[Global mixed performative stability]\label{def:glob}
For $\mu \in \Delta(\Pi)$ define
\[
  \Gglob(\mu) \;:=\; \max_{\pi' \in \Pi}\
  \E_{\pi \sim \mu}\big[V_{M(\pi)}(\pi')\big] \;-\; \E_{\pi\sim\mu}\big[V_{M(\pi)}(\pi)\big].
\]
We say $\mu$ is $\eps$-performatively stable if $\Gglob(\mu) \le \eps$.
\end{definition}

This is the direct analogue of the mixed stability notion of \cite{FP26}. The
comparator $\pi'$ is evaluated inside each induced MDP $M(\pi)$ but does not receive its
own performative effect. When $\mu$ is a point mass, the definition reduces to the notion of $\epsilon$-performatively stable policy in the sense of \cite{MTR23}.

\begin{definition}[Local mixed performative stability]\label{def:loc}
For $\mu \in \Delta(\Pi)$ define
\[
  \Gloc(\mu) \;:=\; \frac{1}{1-\gamma}\,\max_{\pi'\in\Pi}\
  \E_{\pi\sim\mu}\left[\sum_{s\in\gS} d^{\pi}_{P_\pi}(s)\,\Delta_\pi(s,\pi')\right].
\]
We say $\mu$ is $\eps$-\emph{locally} performatively stable if $\Gloc(\mu) \le \eps$.
\end{definition}

The above definition is different than \Cref{def:glob}. By using \Cref{lem:pdl}, we can rewrite \Cref{def:glob} as follows.
$$
\Gglob(\mu) = \frac{1}{1-\gamma} \max_{\pi' \in \Pi}\E_{\pi \sim \mu } \left[\sum_{s \in \gS} d^{\pi'}_{P_\pi}(s) \Delta_\pi(s,\pi')\right]
$$
The only difference from Definition~\ref{def:glob} is that the
per-state advantage terms are weighted by the occupancy of the \emph{deployed} policy
$d^{\pi}_{P_\pi}$ rather than that of the comparator, $d^{\pi'}_{P_\pi}$. This comparison also highlights why we include the factor $\tfrac{1}{1-\gamma}$ in \Cref{def:loc}. This ensures that $\Gloc$ and $\Gglob$ are on the same scale.

\Cref{def:loc} is closely related to the notion of stationary performative stability introduced by \citet{LW24} in the context of non-convex loss functions. Adapted to our setting, a policy $\pi$ is $\delta$-stationary performatively stable if $\norm{V_{M(\pi)}(\pi')|_{\pi'=\pi}}_2 \le \delta$. However, note that the policy $\pi$ is constrained to be in the set $\Pi$, so the natural adaptation of the definition for our setting is the following.

\begin{definition}[Stationary performative stability]\label{def:sps}
A policy $\pi$ is \emph{stationary performatively stable} if
\[
  \max_{\pi'\in\Pi}\left \langle \nabla_{\pi'}V_{M(\pi)}(\pi')|_{\pi'=\pi}, \pi'-\pi\right \rangle \le \delta,
\]
i.e. if $\pi$ is a $\delta$-approximate first-order stationary point of the MDP it induces.
\end{definition}

\begin{lemma}\label{lem:equivalence-local-stationary}
    A policy $\pi$ is $\delta$-stationary performatively stable iff $\Gloc(\pi) \le \delta$.
\end{lemma}
\begin{proof}
    From the policy gradient functional form~\cite{SMSM99} and \cite{ABJKS26} (Theorem 9.4), we have the following expression of the partial derivative.
    $$
    \frac{\partial V_{M(\pi)}(\pi')}{\partial \pi'(a|s)} = \frac{1}{1-\gamma} \sum_{s \in \gS} \sum_a d^{\pi'}_{M(\pi)}(s) Q^{\pi'}_{M(\pi)}(s,a)
    $$
    Taking inner product with $\pi' - \pi$ we obtain the following.
    \begin{align*}
    \left \langle \nabla_{\pi'} V^{\pi'}_{M(\pi)}\big|_{\pi'=\pi}, \pi' - \pi \right \rangle &= \frac{1}{1-\gamma} \sum_{s \in \gS} d^{\pi}_{P_\pi}(s) \left \langle \pi'(\cdot |s) - \pi(\cdot|s), Q^{\pi}_{M(\pi)} (s,\cdot) \right \rangle\\
    &= \frac{1}{1-\gamma} \sum_{s \in \gS} d^{\pi}_{P_\pi}(s) \Delta_{\pi}(s,\pi')
    \end{align*}
    Therefore, the policy $\pi$ is  $\delta$-stationary performatively stable if and only if $\frac{1}{1-\gamma} \max_{\pi' \in \Pi} \sum_{s \in \gS} d^{\pi}_{P_\pi}(s) \Delta_{\pi}(s,\pi') = \Gloc(\pi) \le \delta$.
\end{proof}

\Cref{lem:equivalence-local-stationary} establishes that for distribution with point mass, the notion of stationary performative stability and local performative stability. Therefore, \Cref{def:loc} can be seen as a generalization of local performative stability \cite{LW24} to mixtures over policies.

\textbf{Separation}. There is a strict separation between global and local mixed performative stability. A mixture of policies $\mu \in \Delta(\Pi)$ can be locally mixed performatively stable i.e. $\Gloc(\mu) = 0$, however, it need not be globally mixed performatively stable i.e. $\Gglob(\mu) > c$ for some universal constant $c > 0$. We consider the  instance shown in \Cref{fig:chain}. 

\subsection{The instance}

\begin{definition}[Survival chain]\label{def:chain}
Fix $H \ge 1$ and $\eps \in (0,1/2]$. The state space is $\{s_1,\dots,s_H\} \cup \{\top,\bot\}$
with $\rho = \delta_{s_1}$, and $\gA = \{1,2\}$. The states $\top$ and $\bot$ are absorbing
with $r(\top,\cdot) = 1$ and $r(\bot,\cdot) = 0$. The reward is $0$ on the chain and is
\emph{the same in every model}. In MDP $M$, the process moves from $s_h$ to $s_{h+1}$ under action $b$
(to $\top$ if $h = H$) with probability $\ell^M(b)$ and to $\bot$ otherwise, where
\[
  \ell^A(1) = 1,\quad \ell^A(2) = 1-2\eps;
  \qquad
  \ell^B(1) = 1-2\eps,\quad \ell^B(2) = 1 .
\]
Write $M^A, M^B$ for the two resulting MDPs.% and $\bar P := \tfrac12(P^A + P^B)$.
\end{definition}

\begin{figure}[!h]
\centering
\begin{tikzpicture}[
  >={Stealth[length=2.2mm,width=1.8mm]},
  chain/.style={draw,circle,minimum size=9.5mm,inner sep=0pt,thick},
  absorb/.style={draw,circle,minimum size=9.5mm,inner sep=0pt,thick,double,double distance=0.6pt,fill=black!6},
  lbl/.style={font=\scriptsize},
  drop/.style={->,densely dashed,gray!75},
  scale=0.8
]
\node[chain]  (s1)  at (0,0)    {$s_1$};
\node[chain]  (s2)  at (2.4,0)  {$s_2$};
\node         (sd)  at (4.8,0)  {$\cdots$};
\node[chain]  (sH)  at (7.2,0)  {$s_H$};
\node[absorb] (top) at (9.6,0)  {$\top$};
\node[absorb] (bot) at (4.8,-2.4) {$\bot$};

\draw[->,thick] (s1) -- node[lbl,above]{$\ell^M(b)$} (s2);
\draw[->,thick] (s2) -- node[lbl,above]{$\ell^M(b)$} (sd);
\draw[->,thick] (sd) -- node[lbl,above]{$\ell^M(b)$} (sH);
\draw[->,thick] (sH) -- node[lbl,above]{$\ell^M(b)$} (top);

\draw[drop] (s1) -- node[lbl,black,sloped,above,pos=0.52]{$1-\ell^M(b)$} (bot);
\draw[drop] (s2) -- (bot);
\draw[drop] (sH) -- node[lbl,black,sloped,above,pos=0.52]{$1-\ell^M(b)$} (bot);

\draw[->,thick] (top) to[out=55,in=125,looseness=7] (top);
\draw[->,thick] (bot) to[out=-125,in=-55,looseness=7] (bot);

\node[lbl,right=1.5mm of top] {$r=1$};
\node[lbl,left=1.5mm of bot]  {$r=0$};

\node[draw,rounded corners,inner sep=5pt,font=\scriptsize] at (9.4,-2.4)
  {$\begin{array}{l@{\qquad}l}
      \ell^{A}(1)=1 & \ell^{A}(2)=1-2\eps\\[1pt]
      \ell^{B}(1)=1-2\eps & \ell^{B}(2)=1
    \end{array}$};
\end{tikzpicture}
\caption{The survival chain of Definition~\ref{def:chain} with $H$ stages. From each $s_h$ the
process either advances one step, with the continuation probability $\ell^M(b)$ of the action $b$
taken, or drops to the absorbing failure state $\bot$; the reward is $0$ on the chain. The only
reward is the unit stream collected on absorption in $\top$, so a policy's value is
$\kappa=\gamma^{H}/(1-\gamma)$ times its probability of surviving all $H$ stages. The two models
differ only in which action is safe, and the environment map $M(\cdot)$ of
Definition~\ref{def:map} assigns to each policy whichever model it survives less well in. %The
%certificate splits its mass between the two policies that commit to a single action throughout,
%and beats every single policy because survival is multiplicative across stages.
}
\label{fig:chain}
\end{figure}
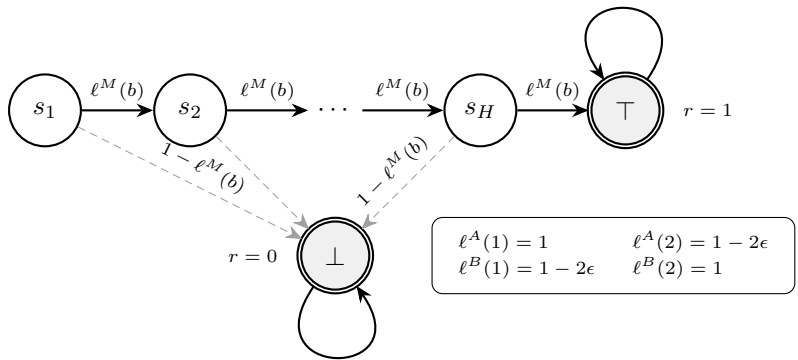

%Each $(s_h,b)$ has $\|P^A(\cdot\mid s_h,b) - \bar P(\cdot\mid s_h,b)\|_1 = 2\eps$, so
%Definition~\ref{def:chain} satisfies Assumption~\ref{ass:range} with $\eps_P = 2\eps$; we set
%$\eps := \eps_P/2$ throughout. 
A policy is described by $y_h := \pi(1 | s_h) \in [0,1]$, and let us write
 $\kappa := \sum_{h\ge H} \gamma^h =\gamma^H/(1-\gamma)$. Then one can verify that
\begin{equation}\label{eq:FAFB}
  V_{M^A}(\pi) = \kappa\ F_A(y), \ F_A(y) := \prod_{h=1}^{H}\big(1 - 2\eps(1-y_h)\big),
  \
  V_{M^B}(\pi) = \kappa\, F_B(y), \ F_B(y) := \prod_{h=1}^{H}\big(1 - 2\eps\, y_h\big)
\end{equation}
since the value is $\kappa$ times the probability of surviving all $H$ transitions.

\begin{definition}[The environment map and the policy class]\label{def:map}
%\begin{enumerate}
 The environment map $M(\pi) := M^A$ if $F_A(y) \le F_B(y)$ and $M(\pi) := M^B$ otherwise. The environment
selects the model in which the deployed policy fares worse.
%\item The policy class $\Pi = \set{\pi^+_\eta, \pi^-_\eta: \eta \in [0,1/2]}$. Here $\pi^+_\eta$ is a stationary policy with $y_h = \pi(1|s_h) = \frac{1}{2} + \eta$. And, similarly $\pi^-_\eta$ is a stationary policy with $y_h = \pi(1|s_h) = \frac{1}{2} - \eta$.
%\end{enumerate}
 %Let $\nu := \tfrac12\delta_{\pi^A}
%+ \tfrac12\delta_{\pi^B}$, where $\pi^A$ plays action $1$ at every $s_h$ and $\pi^B$ plays
%action $2$ at every $s_h$.
\end{definition}

All the missing proofs are deferred to the appendix. 
\begin{lemma}\label{lem:zero-local}
    For the instance described in \Cref{def:chain} and \Cref{def:map}, we have
    $
    \inf_{\mu \in \Delta(\Pi)} \Gloc(\mu) = 0.
    $
\end{lemma}

\begin{lemma}\label{lem:positive-global}
    For the instance described in \Cref{def:chain} and \Cref{def:map}, we have
    $$
    \inf_{\mu \in \Delta(\Pi)} \Gglob(\mu) \ge \frac{1}{16} \frac{\gamma^H}{1-\gamma}  \min \set{1,\eps^2 H^2}.
    $$
    for any $\eps \in [0,1/2]$ and $H \ge 2$.
\end{lemma}

The above two results show a strict separation between local and global mixed performative stability. It is possible to achieve exact local mixed performative stability, however, no matter which mixture is chosen the environment defined in Definitions \ref{def:chain} and \ref{def:map} do not admit a global mixed performatively stable mixture beyond a constant threshold. We note that, when constructing the chain instance, we need $H \ge 2$. This is because for $H=1$, the value function is linear in the policy parameter and the notion of local and global performative stability coincide.

\section{Local stability without any assumption}
\label{sec:local-stability}
We now propose a simple algorithm that outputs a local  performatively stable mixture of policies. The algorithm is based on per-state hedge algorithm. The policy for each state is updated using the Hedge algorithm with time-vayring state, action functions. In particular, at round $t$ the learner deploys
$\pi_t$, the environment changes to $M_t = M(\pi_t)$, and the learner observes the
action-value vectors $Q_t(s,\cdot)$ for all $s$, and the occupancy measure of the deployed policy i.e. $w_t(s) = d^{\pi_t}_{P_t}(s)$ for all $s$. Initially, we assume that the vectors $Q_t(s,\cdot)$ and $w_t(\cdot)$ are directly available to the algorithm, and later we will see how to estimate them from samples.

\begin{definition}[Weighted per-state Hedge]\label{alg:A}
Initialise $\pi_1(\cdot\mid s) = \mathrm{Unif}(\gA)$ for all $s$. For $t = 1,\dots,T$:
deploy $\pi_t$, observe $Q_t$ and $w_t$, and set
\begin{equation}\label{eq:algA}
  \pi_{t+1}(a \mid s) \;\propto\; \pi_t(a\mid s)\,\exp\!\big(\eta\, w_t(s)\, Q_t(s,a)\big),
  \qquad s \in \gS,\ a \in \gA .
\end{equation}
\end{definition}

The occupancy weight inside the exponent is important, as it aligns the algorithm's
regret with the occupancy-weighted objective of Definition~\ref{def:loc}.

\begin{theorem}[Unconditional local stability]\label{thm:A}
Let $M(\cdot)$ be an \emph{arbitrary} environment map. Let $\pi_1,\dots,\pi_T$ be generated
by \eqref{eq:algA} and let $\hat\mu := \mathrm{Unif}\{\pi_1,\dots,\pi_T\}$. Then for every
$\eta > 0$,
\[
  \Gloc(\hat\mu) \;\le\; \frac{1}{1-\gamma}\left( \frac{S\ln A}{\eta T} + \frac{\eta}{8(1-\gamma)^2}\right),
\]
and in particular, choosing $\eta = 2(1-\gamma)\sqrt{2S\ln A / T}$, we have, $ \Gloc(\hat\mu) \;\le\; \frac{1}{(1-\gamma)^{2}}\sqrt{\frac{S\ln A}{2T}}$.
\end{theorem}

\subsection{A finite-sample algorithm}\label{sec:fs}

Algorithm~\ref{alg:A} assumes exact access to $Q_t$ and $w_t$. We now replace both by
estimates built from trajectories collected after the round-$t$ deployment. In particular, we will assume that once $\pi_t$ has been deployed and the environment $M_t = M(\pi_t)$ is realised, the learner 
may draw independent trajectories from $M_t$, starting from either the distribution $\rho$ or at a chosen pair
$(s,a)$, and following $\pi_t$ thereafter. We will assume access to an \emph{occupancy sample} and a \emph{value sample}. 
%It is reasonable when the
%performative response is slow relative to the sampling (a user cohort whose preferences have
%already shifted, a simulator calibrated post-deployment) and unreasonable when every probe is
%itself a deployment. Section~\ref{sec:fs-nosim} discusses what happens without it.

%Both estimators use geometric truncation. Write $\tau \sim \mathrm{Geom}(1-\gamma)$ for the
%law on $\{0,1,2,\dots\}$ with $\Pr[\tau = j] = (1-\gamma)\gamma^{j}$.

\begin{definition}[Occupancy and value samples]\label{def:estimators} Let us write $\tau \sim \mathrm{Geom}(1-\gamma)$ i.e. $\Pr[\tau = j] = (1-\gamma)\gamma^{j}$.
\begin{enumerate}
    \item An \emph{occupancy sample} draws $\tau$, runs $\pi_t$ from $\rho$ in $M_t$ for $\tau$ steps
and returns the visited state $s_\tau$.
\item  A \emph{value sample at $(s,a)$} draws $\tau$, plays $a$ at
$s$ and follows $\pi_t$ thereafter, and returns $\tfrac{1}{1-\gamma}r_\tau$, where $r_\tau$ is
the reward collected from the-$\tau$ step of that rollout.
\end{enumerate}

\end{definition}

%The range is the same as that of $Q_t$ itself, which is what keeps Lemma~\ref{lem:hedge}
%applicable without modification.

\begin{algorithm}[!h]
    \caption{Finite-sample weighted per-state Hedge}\label{alg:finite-sample-hedge}
        \hspace*{\algorithmicindent} 
      \textbf{Input}: Batch sizes $n, k \ge 1$ and a step size $\eta > 0$.
        \begin{algorithmic}[1]
        \State Initialize $\pi_1(\cdot|s) = \mathrm{Unif}(\gA)$ for each $s \in \gS$.
        \For{$t=1,2,\ldots,T$}
            \State Deploy $\pi_t$.
            \State \emph{Occupancy batch.} Draw $n$ independent occupancy samples and let $\widehat w_t$ be
      their empirical distribution on $\gS$.
            \State \emph{Value batch.} Using fresh randomness, for each $s$ with $\widehat w_t(s) > 0$ and each
      $a \in \gA$, draw $k$ independent value samples at $(s,a)$ and let $\widehat Q_t(s,a)$ be
      their average. Set $\widehat Q_t(s,\cdot) := 0$ when $\widehat w_t(s) = 0$.
            \State Update, with $\widehat g_t(s,a) := \widehat w_t(s)\widehat Q_t(s,a)$,
      \begin{equation}\label{eq:algE}
        \pi_{t+1}(a\mid s) \;\propto\; \pi_t(a\mid s)\exp\big(\eta\,\hat g_t(s,a)\big).
      \end{equation}
        \EndFor
         \end{algorithmic}
         \hspace*{\algorithmicindent} \textbf{Output}: $\text{Unif}\left(\pi_1,\ldots,\pi_T\right)$.
\end{algorithm}

\begin{theorem}[Finite-sample local stability]\label{thm:E}
Suppose \Cref{alg:finite-sample-hedge} is run for $T \ge \frac{S \ln(SA/\delta)}{2/k + 1/n}$ iterations, and let
$\pi_1,\dots,\pi_T$ be generated by \eqref{eq:algE} and
$\hat\mu := \mathrm{Unif}\{\pi_1,\dots,\pi_T\}$. Then for any $\eta > 0$ and $\delta\in(0,1)$,
with probability at least $1-\delta$,
\[
  \Gloc(\hat\mu) \;\le\;
  \frac{S \ln A}{\eta T} + \frac{\eta}{8(1-\gamma)^2} + \frac{4}{(1-\gamma)^2} \sqrt{\frac{(2/k + 1/n) S \ln(SA/\delta)}{T} } ,
\]
and with $\eta = 2(1-\gamma)\sqrt{2S\ln A/T}$,
$\Gloc(\hat\mu) \le \frac{6}{(1-\gamma)^{2}} \sqrt{\frac{(2/k+1/n) SA \ln(SA/\delta)}{T}}$. The total number of value and occupancy samples required by the algorithm is $Tn(1+kA)$.
\end{theorem}

 By substituting $n=k=\Theta(1)$ we can get the following corollary.
\begin{corollary}
    In order to reach $\Gloc(\widehat{\mu}) \le \eps$ with probability at least $1-\delta$ it suffices to take $T = O\left(\frac{S \ln(SA/\delta)}{(1-\gamma)^4 \eps^2} \right)$ iterations and $O\left(\frac{S A\ln(SA/\delta)}{(1-\gamma)^4 \eps^2} \right)$ trajectories in total.
\end{corollary}

\section{Global Stability Under Bounded Transition Range}
We now consider the problem of generating a global performatively stable mixture of policies. We will make the following assumption.
\begin{assumption}[Bounded transition range]\label{ass:range}
There exist a reference transition kernel $\bar P$ and $\eps_P \ge 0$ such that
\[
  \max_{s,a}\ \big\| P_\pi(\cdot \mid s,a) - \bar P(\cdot \mid s,a)\big\|_1 \;\le\; \eps_P
  \qquad \text{for all } \pi \in \Pi .
\]
\end{assumption}

Assumption~\ref{ass:range} constrains only the \emph{range} of the induced kernels, and not the
regularity of $\pi \mapsto P_\pi$: the map may be arbitrarily discontinuous within that
range. It is implied by, but strictly weaker than, the $\eps$-sensitivity condition used in
the performative RL literature (Assumption 1, \cite{MTR23}). In particular, if
$\sup_{\pi,\pi'}\max_{s,a}\|P_\pi(\cdot\mid s,a) - P_{\pi'}(\cdot\mid s,a)\|_1 \le \tilde{\eps} \norm{\pi - \pi'}_1$
then Assumption~\ref{ass:range} holds with $\bar P := P_{\pi_1}$ and $\eps_P = 2\tilde{\eps}S$. Note that, we place no assumption on the reward map $\pi \mapsto r_\pi$.

\begin{definition}[Unweighted per-state Hedge]\label{alg:B}
Initialise $\pi_1(\cdot| s) = \mathrm{Unif}(\gA)$. For $t = 1,\dots,T$, deploy $\pi_t$,
observe $Q_t$, and set
\begin{equation}\label{eq:algB}
  \pi_{t+1}(a\mid s) \;\propto\; \pi_t(a\mid s)\exp\!\big(\eta\, Q_t(s,a)\big).
\end{equation}
\end{definition}

\begin{theorem}[Global stability up to the transition range]\label{thm:global-stability}
Suppose \Cref{ass:range} holds. Let
$\pi_1,\dots,\pi_T$ be generated by \eqref{eq:algB} with
$\eta = 2(1-\gamma)\sqrt{2\ln A/T}$, and let $\hat\mu := \mathrm{Unif}\{\pi_1,\dots,\pi_T\}$.
Then
$$
  \ \Gglob(\hat\mu)\;\le\; \frac{1}{(1-\gamma)^{2}}\sqrt{\frac{\ln A}{2T}}
  \;+\; \frac{\gamma\,\eps_P}{(1-\gamma)^{3}}.
$$
\end{theorem}

Note that, in contrast to Theorem~\ref{thm:A}, \Cref{thm:global-stability} can guarantee convergence to a globally performatively stable mixture up to a floor of $\Theta(\gamma\eps_P/(1-\gamma)^3)$. In the next section, we will observe that there exists a lower bound preventing convergence to an arbitrary precision, presenting a sharp contrast to similar results in performative prediction~\cite{FP26}.

Theorems~\ref{thm:A} and \ref{thm:global-stability} use \emph{different} algorithms: \eqref{eq:algA}
weights the exponent by $w_t(s)$, \eqref{eq:algB} does not. Unweighted Hedge does not yield
Theorem~\ref{thm:A}, since bounding $\sum_t w_t(s)\Delta_t(s,\pi')$ with time-varying $w_t$
is exactly the difficulty of Section~\ref{sec:local-stability}. Weighted Hedge does not obviously
yield Theorem~\ref{thm:global-stability}, since \eqref{eq:perstate} carries the factor $w_t(s)$ that the
global objective does not supply. Whether a single dynamic attains both guarantees
simultaneously is an open question. Finally, observe that that neither algorithm needs to know $\bar P$ or $\eps_P$.

\subsection{A Finite-Sample Algorithm}
Similar to \Cref{alg:finite-sample-hedge} we can develop an algorithm that obtains a globally stable mixture of policies using trajectories drawn from deployed policies. 
\begin{algorithm}[!h]
    \caption{Finite-sample unweighted per-state Hedge}\label{alg:finite-sample-unweighted-hedge}
        \hspace*{\algorithmicindent} 
      \textbf{Input}: Batch sizes $k \ge 1$ and a step size $\eta > 0$.
        \begin{algorithmic}[1]
        \State Initialize $\pi_1(\cdot|s) = \mathrm{Unif}(\gA)$ for each $s \in \gS$.
        \For{$t=1,2,\ldots,T$}
            \State Deploy $\pi_t$.
            \State \emph{Value batch.} For each state $s \in \gS$ and action
      $a \in \gA$, draw $k$ independent value samples at $(s,a)$ and let $\widehat Q_t(s,a)$ be
      their average.
            \State Update,
      \begin{equation}\label{eq:update-unweighted-hedge}
        \pi_{t+1}(a\mid s) \;\propto\; \pi_t(a\mid s)\exp\big(\eta\,\widehat Q_t(s,a)\big).
      \end{equation}
        \EndFor
         \end{algorithmic}
         \hspace*{\algorithmicindent} \textbf{Output}: $\text{Unif}\left(\pi_1,\ldots,\pi_T\right)$.
\end{algorithm}

\begin{theorem}[Finite-sample global stability]\label{thm:finite-global}
Suppose \Cref{alg:finite-sample-unweighted-hedge} is run with $\eta = 2(1-\gamma)\sqrt{2S\ln A/T}$ and $T \ge k \ln(SA/\delta)$ iterations, and let
$\pi_1,\dots,\pi_T$ be generated by \eqref{eq:update-unweighted-hedge} and
$\hat\mu := \mathrm{Unif}\{\pi_1,\dots,\pi_T\}$. Then 
with probability at least $1-\delta$,
\[
  \Gglob(\hat\mu) \;\le\;
  \frac{1}{(1-\gamma)^{2}}\sqrt{\frac{\ln (SA)}{2T}} + \frac{3\gamma\,\eps_P}{(1-\gamma)^3} + \frac{4}{(1-\gamma)^2} \sqrt{\frac{  \ln(SA/\delta)}{kT} }.
\]
 The total number of value and occupancy samples required by the algorithm is $TSAk$.
\end{theorem}
By substituting $k=\Theta(1)$ we can get the following corollary.
\begin{corollary}
    In order to reach $\Gglob(\widehat{\mu}) \le \eps$ for any $\eps \ge \frac{3\gamma\eps_P}{(1-\gamma)^3}$ with probability at least $1-\delta$ it suffices to take $T = O\left(\frac{ \ln(SA/\delta)}{(1-\gamma)^4 \eps^2} \right)$ iterations and $O\left(\frac{S A\ln(SA/\delta)}{(1-\gamma)^4 \eps^2} \right)$ trajectories in total.
\end{corollary}

\section{Lower Bound on Global Stability}
In the previous section, we have shown that under the bounded transition range assumption, it is possible to obtain a globally performatively stable mixture of policies up to a floor of $\Theta(\gamma\eps_P/(1-\gamma)^3)$. Moreover, substituting $H=2$ in \Cref{lem:positive-global} shows that there exists an instance where the global performative stability i.e. $\Gglob(\mu)$is lower bounded by $\Omega(\gamma^2\eps^2_P/(1-\gamma))$ for any mixture $\mu$. The reason the lower bound instance in \Cref{def:chain} generates a weaker bound is that the learning algorithm receives full feedback about the new environment. We will consider a feedback model that closely resembles the setup considered in Algorithms \ref{alg:finite-sample-hedge} and \ref{alg:finite-sample-unweighted-hedge}.

\begin{definition}[Local feedback]\label{def:localfb}
At round $t$ the learner deploys $\pi_t$, the environment becomes $M_t = M(\pi_t)$, and the
learner draws trajectories from $M_t$ under $\pi_t$. Each trajectory is of the form $(s_t^0,a_t^0,r_t^0,s_t^1,a_t^1,r_t^1,\dots)$, and the learner never observes reward at unvisited pairs $(s,a)$, nor $M(\pi)$ for any $\pi$ it has not deployed. The process repeats for $T$ round, and 
after $N$ trajectories in total the learner outputs
$\hat\mu = \mathrm{Unif}\{\pi_1,\dots,\pi_T\}$.
\end{definition}

%To state the bound we fix the protocol precisely. Everything below is on a single probability
%space carrying the environment randomness and the algorithm's internal randomisation.
In order to prove the lower bound, we will consider a family of instances indexed by $j \in [K]$. Each instance is a simple two-action MDP with a single state, where one of the actions is hidden and contrarian. The environment rewards the hidden action exactly when the learner underweights it. 
\begin{definition}[Hidden-Action Instance]\label{def:hidden}
Fix  $\eps \in (0,1/4]$ and a hidden action $j \in [K]$. The state space is
$\{s_0,\top,\bot\}$ with $\rho = \delta_{s_0}$ and $\gA = [K]$. The states $\top,\bot$ are absorbing with
$r(\top,\cdot) = 1$, $r(\bot,\cdot) = 0$, and $r(s_0,\cdot) = 0$. From $s_0$, action $a$ moves
to $\top$ with probability $p_a$ and to $\bot$ otherwise, where $p_a = \tfrac12$ for $a \ne j$
and
\[
  p_j(\pi) \;=\; \begin{cases} \tfrac12 + \eps, & \pi(j| s_0) < \tfrac12,\\[2pt]
                               \tfrac12 - \eps, & \pi(j| s_0) \ge \tfrac12 .\end{cases}
\]
\end{definition}

Note that, writing $\eps_P := 2\eps$, the instance satisfies
Assumption~\ref{ass:range} with reference kernel $\bar P(\cdot| s_0,a) = \mathrm{Bern}(\tfrac12)$
for \emph{every} $a$. The next theorem shows that if the algorithm is restricted to sample at most $O(K/\eps_P^2)$ trajectories, then there exists a hidden action $j$ for which the expected global performative stability of the output mixture $\hat\mu$ is lower bounded by $\Omega(\gamma\eps_P/(1-\gamma))$. This shows that the upper bound in Theorem~\ref{thm:global-stability} is tight up to a factor of $O(1/(1-\gamma)^2)$.

\begin{theorem}\label{thm:linear-lower-bound}
Consider the instance of Definition~\ref{def:hidden} with $K \ge 20$ and
$\eps = \eps_P/2 \le 1/4$. Let any algorithm output a mixture of policies
$\hat\mu$. If the total number of trajectories $N$ satisfies
$
  N \;\le\; \frac{K}{24\,\eps_P^{2}},
$
then
%\[
%  \frac{1}{K}\sum_{j=1}^{K}\E_j\Big[\Gglobj{j}(\hat\mu)\Big] \;\ge\; \frac{\gamma\,\eps_P}{8\,(1-\gamma)},
%\]
there exists a hidden action $j\in[K]$ for which
$$\E_{\mathbb{P}_j}\big[\Gglob(\hat\mu)\big] \ge \frac{\gamma\eps_P}{8(1-\gamma)},$$
where $\mathbb{P}_j$ is the law of the algorithm's trajectory of observations with hidden action $j$.
\end{theorem}

\section{Multi-Agent Setting}
We now consider a multi-agent generalization of the framework of performative reinforcement learning~\cite{SSYMR25}. There are $n$ players holding policies $\bpi = (\pi^1,\dots,\pi^n)$ and the environment
responds to the whole profile of policies. In this section, we aim to generalize the results of \Cref{thm:A} (unconditional local stability) and \Cref{thm:global-stability} (conditional global stability) from single-agent to multi-agent settings. Let us first define the notion of performative Markov game.

\subsection{Performative Markov games}

Let $\gA_1,\dots,\gA_n$ be finite action sets, $A_i := |\gA_i|$, $\gA := \prod_i \gA_i$, and
$\Pi^i := \Delta(\gA_i)^{\gS}$, $\bm\Pi := \prod_i \Pi^i$. A performative Markov game is an
arbitrary map
\[
  M : \bm\Pi \longrightarrow \big\{(P, r^1,\dots,r^n)\big\},
  \qquad \bpi \longmapsto M(\bpi) = \big(P_{\bpi},\, r^1_{\bpi},\dots,r^n_{\bpi}\big),
\]
with $P(\cdot\mid s,\bm a) \in \Delta(\gS)$ and $r^i : \gS\times\gA \to [0,1]$. Therefore, a policy profile $\bpi$ of the $n$ players results in a Markov game with transition $P_{\bpi}$ and utility profile $(r^1_{\bpi},\dots,r^n_{\bpi})$. The occupancy measure $d^{\bpi}_{P}$ is defined as before with respect to the joint action $\bm a \in \gA$. Note that, there is a single state occupancy, shared by all players. 

We will write $Q^i_M(s,\bm a)$ for the action-value of player $i$ in the Markov game $M$ at state $s$ and joint action $\bm a$. We can also define the marginial action-value of player $i$ against the others' current play as
\[
  Q^{i,\bpi}_M(s,a^i) \;:=\; \E_{a^{-i}\sim\bpi^{-i}(\cdot\mid s)}\big[Q^{i}_M\big(s,(a^i,a^{-i})\big)\big] .
\]
Additionally, we will write $V^i_M(\pi^i,\bpi^{-i})$ for the value of player $i$ in the Markov game $M$ when it plays $\pi^i$ and the other players play $\bpi^{-i}$. We can now extend the definitions of global and local performative stability to the multi-agent setting. Note that, the mixtures $\mu \in \Delta(\bm\Pi)$ are over \emph{joint profiles} and are therefore correlated across
players. 

\begin{definition}[Multi-Agent Mixed performative stability]\label{def:multi-agent-stability}
For $\mu \in \Delta(\bm\Pi)$ set
\begin{align*}
  \Gglobm{i}(\mu) &:= \max_{\pi'\in\Pi^i}\ \E_{\bpi\sim\mu}\left[V^i_{M(\bpi)}\big(\pi',\bpi^{-i}\big)\right] - \E_{\bpi \sim \mu}\left[V^i_{M(\bpi)}(\bpi)\right],\\
  \Glocm{i}(\mu)  &:= \frac{1}{1-\gamma}\max_{\pi'\in\Pi^i}\ \E_{\bpi\sim\mu}\sum_{s} d^{\bpi}_{P_{\bpi}}(s)\,
     \big\langle \pi'(\cdot\mid s) - \pi^i(\cdot\mid s),\ Q^{i,\bpi}_{M(\bpi)}(s,\cdot)\big\rangle,
\end{align*}
and $\Gglob(\mu) := \max_i \Gglobm{i}(\mu)$, $\Gloc(\mu) := \max_i \Glocm{i}(\mu)$.
\end{definition}

As in the single-agent case the deviating policy $\pi'$ is scored inside the induced environment $M(\bpi)$ and
does not receive its own performative effect. 
%Both
%solution concepts are therefore coarse, and here coarseness is doubly forced --- once by
%performativity and once by the game, since even for a fixed stochastic game no-regret dynamics
%yield correlated rather than Nash equilibria.

%So $\mu$ is locally stable exactly when, at every state, the occupancy-weighted play forms an
%approximate coarse correlated equilibrium of the one-shot game with payoffs
%$Q^{i,\bpi}_{M(\bpi)}(s,\cdot)$ --- that is, when $\mu$ is an approximate \emph{Markov CCE} of the
%induced game. Definition~\ref{def:multi} is thus not a new object invented for performativity:
%it is the standard stationary solution concept for stochastic games, evaluated in the
%elf-induced environment.

\subsection{Multi-Agent Local Stability without any assumption}

\begin{definition}[Independent weighted per-state Hedge]\label{alg:G}
Player $i$ initialises $\pi^i_1(\cdot| s) = \mathrm{Unif}(\gA_i)$. At round $t$ the profile
$\bpi_t = (\pi^1_t,\dots,\pi^n_t)$ is deployed; each player observes its own $Q^i_t(s,\cdot) := Q^{i,\bpi_t}_{M_t}(s,\cdot)$
and the common occupancy measure $w_t(s) := d^{\bpi_t}_{P_t}(s)$, and updates
\begin{equation}\label{eq:algG}
  \pi^i_{t+1}(a| s)\;\propto\;\pi^i_t(a| s)\exp\big(\eta_i\,w_t(s)\,Q^i_t(s,a)\big).
\end{equation}
%No player observes the others' policies, rewards, or action sets.
\end{definition}

\begin{theorem}[Unconditional local stability, $n$ players]\label{thm:multi-agent-local-stability}
Let $M(\cdot)$ be an arbitrary performative Markov game. Let $\bpi_1,\dots,\bpi_T$ be generated
by \eqref{eq:algG} with $\eta_i = 2(1-\gamma)\sqrt{2S\ln A_i/T}$ and let
$\hat\mu := \mathrm{Unif}\{\bpi_1,\dots,\bpi_T\}$. Then for every player $i$,
\[
  \Glocm{i}(\hat\mu)\;\le\;\frac{1}{(1-\gamma)^2}\sqrt{\frac{S\ln A_i}{2T}}, \text{ and }
\ \Gloc(\hat\mu)\;\le\;\frac{1}{(1-\gamma)^{2}}\sqrt{\frac{S\ln A_{\max}}{2T}}
\]
with $A_{\max} := \max_i A_i$. %The bound involves no assumption on $M(\cdot)$, no condition
%relating players' sensitivities, and \emph{no dependence on $n$}.
\end{theorem}

Similar to \Cref{thm:A}, the bound in Theorem~\ref{thm:multi-agent-local-stability} is unconditional, and holds for any performative Markov game. Note that, convergence in multiplayer performative prediction~\cite{NFDFR23} requires conditions about the joint sensitivity of the environment map, and also about the structure of the game (e.g. convexity). In contrast, Theorem~\ref{thm:multi-agent-local-stability} avoids such requirements by considering a mixture over strategy profiles i.e. the mixture $\widehat{\mu}$ is correlated across the players. 

Moreover, the bound depends on $A_{\max}$, not on the joint action space $|\gA| = \prod_i A_i$. The reason is that the argument never enumerates joint actions, because player $i$'s benchmark is its own unilateral deviation and its gains are already marginalised over $\bpi^{-i}_t$. 

\subsection{Global Stability for Markov Potential Games}\label{sec:mpg-global}

The situation for the global notion of stability is qualitatively different, and the reason is visible in
the proof of Theorem~\ref{thm:global-stability}. The main argument was that the
comparator's occupancy $d^{\pi'}_{\bar P}$ is a \emph{fixed} weight vector, so that a uniform
per-state regret bound could be integrated against it. However, in a Markov game, player $i$'s deviation
produces the profile $(\pi', \bpi^{-i}_t)$, whose occupancy
$d^{(\pi',\,\bpi^{-i}_t)}_{\bar P}$
depends on $t$ through the opponents' policies, however small $\eps_P$ is. Therefore, the weights are
 time-varying even at $\eps_P = 0$, and the proof of \Cref{thm:global-stability} cannot be generalized. In fact, we next construct a simple two-player Markov game that separates local and global stability, even in the absence of performative effects ($\eps_P = 0$). 
 
 %We make additional structural assumption on the induced Markov games to recover global stability. 

\begin{definition}[Two-player survival chain]\label{def:2psurvival}
Fix $H\ge2$ and $\eps\in(0,1/2)$. States $\{s_1,\dots,s_H\}\cup\{\top,\bot\}$ with
$\rho=\delta_{s_1}$; $\gA_1 = \{1,2\}$ and $\gA_2 = \{A,B\}$. From $s_h$ under the joint action
$(j,b)$ the process advances to $s_{h+1}$ (to $\top$ if $h=H$) with probability $\ell^{b}(j)$ and
falls to $\bot$ otherwise, where $\ell^{A} = (1,\,1-2\eps)$ and $\ell^{B} = (1-2\eps,\,1)$. The
rewards are $r^1(\top,\cdot)=1$, $r^1(\bot,\cdot)=0$ and $r^2 = 1-r^1$, zero on the chain. The
game is \emph{fixed}, so $\eps_P=0$.
\end{definition}

This example generalizes \Cref{def:chain} and player $2$ chooses which of player $1$'s actions is safe --- precisely the choice that the
environment map made in Definition~\ref{def:map}.

\begin{proposition}[Separation]\label{prop:multisep}
Let $\bar\pi^1$ play $\mathrm{Unif}(\gA_1)$ at every $s_h$, let $\pi^2_A$ play $A$ at every state
and $\pi^2_B$ play $B$ at every state, and set
$
  \mu^\star \;:=\; \tfrac12\,\delta_{(\bar\pi^1,\,\pi^2_A)} \;+\; \tfrac12\,\delta_{(\bar\pi^1,\,\pi^2_B)} .
$
Then,
\[
  \Gloc(\mu^\star) \;=\; 0
  \qquad\text{and}\qquad
  \Gglob(\mu^\star) \;\ge\; \tfrac{\gamma^H}{16(1-\gamma)} \min\set{1,\eps^2H^2} \ \forall H \ge 2 .
\]
\end{proposition}

% \begin{openq}[Separation at $\eps_P = 0$]\label{oq:multisep}
% Construct a fixed (non-performative) general-sum stochastic game and a mixture that is an exact
% Markov CCE --- i.e. $\Gloc(\mu) = 0$ --- but has $\Gglob(\mu) = \Omega(1)$, where the global
% deviation accounts for the change in occupancy that a deviating player induces. We expect such
% a separation to exist, since it is the multi-agent shadow of \Cref{lem:positive-global}; unlike that
% theorem it would require no performativity at all. A positive answer would say that in
% multi-agent performative RL the local notion is forced by the game rather than by the
% environment map.
% \end{openq}

\begin{assumption}[Performative Markov potential game]\label{ass:mpg}
For every profile $\bpi \in \bm\Pi$ the game $M(\bpi)$ admits a potential
$\Phi_{M(\bpi)} : \bm\Pi \to [0,\tfrac{1}{1-\gamma}]$, such that for all $i$, all
$\sigma,\sigma' \in \Pi^i$ and all $\bm\tau^{-i}$,
\[
  V^i_{M(\bpi)}(\sigma',\bm\tau^{-i}) - V^i_{M(\bpi)}(\sigma,\bm\tau^{-i})
  \;=\; \Phi_{M(\bpi)}(\sigma',\bm\tau^{-i}) - \Phi_{M(\bpi)}(\sigma,\bm\tau^{-i}).
\]
\end{assumption}

We will make the following assumptions about the potential function $\Phi_{M(\bpi)}$ in order to obtain global stability.
\begin{assumption}[Regularity]\label{ass:mpgreg}
Let $\bm\Pi = \prod_i \Delta(\gA_i)^{\gS}$ be the joint policy space, and let $G  := \sup_{\bpi,\bm\sigma}\norm{\nabla_{\sigma^i}\Phi_{M(\bpi)}(\bm\sigma)}_2$ be the maximum norm of partial derivatives.
\begin{enumerate}
\item[(P1)] \emph{Smoothness.} Each $\Phi_{M(\bpi)}(\cdot)$ is $L_\Phi$-smooth on $\bm\Pi$.
\item[(P2)] \emph{Sensitivity.} $\big|\Phi_{M(\bpi)}(\bm\sigma) - \Phi_{M(\bpi')}(\bm\sigma)\big|
\le \eps_L\,\|\bpi-\bpi'\|_2$ for all $\bm\sigma,\bpi,\bpi'$.
\end{enumerate}
\end{assumption}
The smoothness assumption (P1) is standard in the analysis of gradient-based methods, and the sensitivity assumption (P2) is a natural generalization of Assumption~\ref{ass:range} to the multi-agent setting. We now propose an independent projected policy gradient ascent algorithm for multi-agent performative Markov potential games.
\begin{definition}[Independent projected policy gradient ascent]\label{alg:H}
For each $t\in \set{1,\ldots,T}$ let $\Phi_t(\cdot) := \Phi_{M(\bpi_t)}(\cdot)$ and $\alpha \in (0, 1/L_\Phi]$. For each agent $i$, update
\begin{equation}\label{eq:algH}
  \pi^i_{t+1} \;:=\; \mathrm{proj}_{\bm\Pi^i}\big(\pi_t + \alpha\,\nabla_{\pi^i}\Phi_t(\bpi_t)\big).
\end{equation}
\end{definition}
Equivalenty, since the joint policy space $\bm\Pi$ is a product of the individual policy spaces $\Pi^i$, the update rule can be equivalently written as
\begin{equation}\label{eq:joint-update}
  \bpi_{t+1} \;=\; \mathrm{proj}_{\bm\Pi}\big(\bpi_t + \alpha\,\nabla_{\bpi}\Phi_t(\bpi_t)\big).
\end{equation}

% \begin{lemma}[The $n$ gaps collapse onto one function]\label{lem:mpgcollapse}
% Under Assumption~\ref{ass:mpg}, for every $\mu$ and every $i$,
% \[
%   \Gglobm{i}(\mu) \;=\; \max_{\sigma\in\Pi^i}\ \E_{\bpi\sim\mu}
%   \Big[\Phi_{M(\bpi)}\big(\sigma,\bpi^{-i}\big) - \Phi_{M(\bpi)}(\bpi)\Big].
% \]
% \end{lemma}

% \begin{proof}
% Immediate from Assumption~\ref{ass:mpg} with $\sigma' = \sigma$, $\sigma = \pi^i$,
% $\bm\tau^{-i} = \bpi^{-i}$, taking expectations and maximising.
% \end{proof}

\begin{theorem}\label{thm:convergence-potential-mpg}
  Suppose Assumptions \ref{ass:mpg} and \ref{ass:mpgreg} hold. Let $\bpi_1,\ldots,\bpi_T$ be generated by \Cref{eq:algH} with $\alpha \le 1/L_\phi$, and $\widehat{\mu} := \text{Unif}\set{\bpi_1,\ldots,\bpi_T}$. Then,
  $$
  \Gglob(\widehat{\mu}) \le \max_{\bpi \in \bm \Pi} \max_{i, \sigma \in \bm \Pi} \norm{\tfrac{d^{(\sigma,\bpi^{-i})}_{M(\bpi)}}{d^{\bpi}_{M(\bpi)}}}_\infty \left( G + \frac{\sqrt{2S}}{\alpha} \right) \sqrt{n} \left( 2\alpha \eps_L + \sqrt{\frac{2\alpha}{T(1-\gamma)}}\right).
  $$
  Furthermore, choosing $\alpha \le \min \set{1/L_\phi, \sqrt{2S}/G}$ and $T \ge \tfrac{1}{(1-\gamma)\alpha \eps_L^2}$ we obtain 
  $$
\Gglob(\widehat{\mu}) \le \max_{\bpi \in \bm \Pi} \max_{i, \sigma \in \bm \Pi} \norm{{d^{(\sigma,\bpi^{-i})}_{M(\bpi)}}/{d^{\bpi}_{M(\bpi)}}}_\infty 8 \sqrt{2Sn} \eps_L.
  $$ 
\end{theorem}

\Cref{thm:convergence-potential-mpg} establishes that if the number of iterations $T$ is chosen as $O(\eps_L^{-2})$ then it is possible to attain global stability gap of order $O(\kappa \sqrt{Sn} \eps_L)$, where $\kappa := \norm{{d^{(\sigma,\bpi^{-i})}_{M(\bpi)}}/{d^{\bpi}_{M(\bpi)}}}_\infty $ is the distribution mismatch coefficient. Compared to \Cref{thm:multi-agent-local-stability}, the bound does not depend on the number of actions but grows as square root of the number of agents. The presence of distribution mismtach is typical of analysis of Markov potential games, e.g. \cite{DWZJ22}.

\section{Conclusion and Future Work}\label{sec:conclusion}

We studied performative reinforcement learning through mixtures of policies, and asked how much of the sensitivity assumption of \citet{MTR23} is really needed for converging to a stable mixture. Our answer is that it depends on the notion of stability one
targets. For \emph{local} stability, no assumption is needed and a weighted per-state Hedge converges at an $O(1/\sqrt{T})$ rate, and the same holds
for every player in an $n$-player performative Markov game with no condition on the joint
environment map or on the structure of the game. However, for \emph{global} stability the picture is
different. We require the assumption of \emph{Bounded Transition Range} (\Cref{ass:range}) for converging upto a floor of $O(\gamma\eps_P/(1-\gamma)^3)$. Moreover, our lower bound shows that the $O(\gamma\eps_P/(1-\gamma))$ floor is unavoidable under local trajectory feedback. Global stability is even more challenging in the multi-agent setting, as we show that there exists a fixed two-player game with $\eps_P = 0$ which can separate local and global multi-agent stability for mixtures of policies. Finally,we shwo that imposing potential-game structure provides can recover multi-agent global stability.

Several questions remain open.
\begin{enumerate}[label=(\roman*),leftmargin=2em]
  \item Our two guarantees for local and global stability are attained by two different online algorithms. A \emph{weighted} per-state Hedge converges for local stability, and an \emph{unweighted} per-state Hedge converges up to a non-vanishing floor for global stability. An interesting question is whether a single algorithm can attain both local and global stability simultaneously. 
\item  Regarding global stability, the upper and lower bound matches in terms of $\eps_P$, however, a gap
of $O(1/(1-\gamma)^2)$ in the effective horizon remains. Closing it would require either a sharper
occupancy-transfer argument than the one behind \Cref{thm:global-stability}, or a lower bound
instance with a genuinely long horizon rather than the $O(1)$ horizon of
\Cref{def:hidden}.
\item As we replace determinstic stable policy with a mixture of policies, it is important the mixture is supported on a small number of policies to ensure easy deployment. Our current convergence rate of $O(1/\sqrt{T})$ implies that the mixture is supported on $O(1/\varepsilon^2)$ policies, and an interesting question to understand when it can be improved. One idea would be to show a faster rate of convergence under additional assumptions on the environment map, similar to the assumption ofstrong convexity of the performative loss function~\citep{FP26}.
\item Our sample complexities scale with $SA$ and rely on being able to re-sample the induced
environment after each deployment. Extending our framework to a large class of MDPs e.g. linear MDPS~\citep{MR25} or general function approximation, is an interesting direction of future work. Additionally, it would be interesting to extend our results to protocols where re-sampling is not available --- as in the gradual-adaptation model of
\citet{RTMR24} or the stateful setting of \citet{BHK22}.
\item There are several open questions on the multi-agent side. First, we have shown that local stability is achievable for arbitrary games, and global stability is achievable for performativepotential games. It would be interesting to understand whether there are structural conditions lying between arbitrary games and potential games under which global stability is still achievable. Second, our current analysis of multi-agent stability relies on exact gradients of the induced potential function, and it would be interesting to understand whether a finite-sample counterpart can be established.

\end{enumerate}

\noindent \textbf{Acknowledgement}: We used to Fable 5 to generate the survival chain instance in \Cref{def:chain}, and the counterexample used in \Cref{thm:linear-lower-bound}. Opus 5 was used to search literature and to generate text for some of the sections.

% Uncomment and fill in before posting, if applicable.
% This work was supported by ...

\printbibliography

%\section{Appendix}
\clearpage
\appendix

\part{Appendix}
\localtableofcontents

\section{Missing Proofs}

\subsection{Local Stability of the Survival Chain Instance\label{app:local-stability-survival-chain}}

\begin{proof}[Proof of \Cref{lem:zero-local}]
    In order to prove the first bound, we will consider the following policy class,
    $$\Pi_0 = \set{\pi^+_\eta, \pi^-_\eta: \eta \in [0,1/2]},$$
    where $\pi^+_\eta$ is a stationary policy with $y_h = \pi(1|s_h) = \frac{1}{2} + \eta$. And, similarly $\pi^-_\eta$ is a stationary policy with $y_h = \pi(1|s_h) = \frac{1}{2} - \eta$. Let us also define the mixture $\mu_\eta$ to be supported with equal probability over $\pi^+_\eta$ and $\pi^-_\eta$ i.e. 
    $$\mu_\eta = \frac{1}{2} \cdot \delta_{\pi^+_\eta} + \frac{1}{2} \cdot \delta_{\pi^-_\eta}.$$

    First note that, $V_{M^A}(\pi^+_\eta) = \kappa \left( 1-2\eps(1-(1/2+\eta))\right)^H = \kappa (1-\eps + 2\eps \eta)^H$, and $V_{M^B}(\pi^+_\eta) = \kappa (1-2\eps(1/2+\eta))^H = \kappa (1-\eps - 2\eps \eta)^H$. Since $F_A(\pi^+_\eta) \ge F_B(\pi^+_\eta)$ we have $M(\pi^+_\eta) = M^B$. Similarly, we can show that $M(\pi^-_\eta) = M^A$. 
    
    \textbf{Local Gap}. Now we compute $\Gloc(\mu_\eta)$ using \Cref{prop:cce}. Let us evaluate the following two quantities. 
    \[
  \bar q_\mu(s,a) := \E_{\pi\sim\mu}\big[d^{\pi}_{P_\pi}(s)\,Q^{\pi}_{M(\pi)}(s,a)\big],
  \qquad
  \bar v_\mu(s) := \E_{\pi\sim\mu}\big[d^{\pi}_{P_\pi}(s)\,\langle \pi(\cdot\mid s), Q^{\pi}_{M(\pi)}(s,\cdot)\rangle\big].
\]
Let us write $\bar{\ell}(\eta) = 1-\eps - 2\eps \eta$. Then we have,
$$
Q^{\pi^+_\eta}_{M^B}(s_h,a) = \ell^{B}(a) \cdot V^{\pi^+_\eta}_{M^B}(s_{h+1}) = \ell^B(a) \cdot \kappa \cdot \bar{\ell}(\eta)^{H-h}
\text{ and } d^{\pi^+_\eta}_{P^B}(s_h) = \bar{\ell}(\eta)^h.
$$
Similarly,
$$
Q^{\pi^-_\eta}_{M^A}(s_h,a) = \ell^{A}(a) \cdot V^{\pi^+_\eta}_{M^A}(s_{h+1}) = \ell^A(a) \cdot \kappa \cdot \bar{\ell}(\eta)^{H-h}
\text{ and } d^{\pi^-_\eta}_{P^A}(s_h) = \bar{\ell}(\eta)^h.
$$
As $\ell^A(a) + \ell^B(a) = 2-2\eps$ for any action $a \in \gA$, we obtain
$$
\bar{q}_{\mu_\eta}(s,a) = \frac{1}{2} Q^{\pi^+_\eta}_{M^B}(s,a) \cdot d^{\pi^+_\eta}_{P^B}(s) + \frac{1}{2} Q^{\pi^-_\eta}_{M^A}(s,a) \cdot d^{\pi^-_\eta}_{P^A}(s) = \kappa \cdot (1-\eps) \bar{\ell}(\eta)^H.
$$
Therefore, $\bar{q}_{\mu_\eta}(s,a)$ is independent of action $a$. In order to evaluate $\bar{v}_\mu(s)$, let us first evaluate
$\big\langle \pi^+_\eta(\cdot|s_h), Q^{\pi^+_\eta}_{M^B}(s_h,\cdot)\big\rangle$.
\begin{align*}
    \left \langle \pi^+_\eta(\cdot|s_h), Q^{\pi^+_\eta}_{M^B}(s_h,\cdot)\right \rangle &= \left( \frac{1}{2} + \eta \right)  \ell^B(1) \cdot \kappa \cdot \bar{\ell}(\eta)^{H-h} +  \left( \frac{1}{2} - \eta \right)  \ell^B(2) \cdot \kappa \cdot \bar{\ell}(\eta)^{H-h} \\
    &= \kappa \cdot \bar{\ell}(\eta)^{H-h+1}
\end{align*}
Therefore,
\begin{align*}
    \bar{v}_\mu(s_h) &= \frac{1}{2} d^{\pi^+_\eta}_{P^B}(s)  \left \langle \pi^+_\eta(\cdot|s_h), Q^{\pi^+_\eta}_{M^B}(s_h,\cdot)\right \rangle  + \frac{1}{2} d^{\pi^-_\eta}_{P^A}(s)  \left \langle \pi^-_\eta(\cdot|s_h), Q^{\pi^-_\eta}_{M^A}(s_h,\cdot)\right \rangle \\
    &= \kappa \cdot \bar{\ell}(\eta)^{H+1}.
\end{align*}
We now use \Cref{prop:cce} to bound $\Gloc(\mu_\eta)$.
\begin{align*}
    \Gloc(\mu_\eta) &= \frac{1}{1-\gamma}\sum_{s\in\gS}\Big( \max_{a \in \gA} \bar q_\mu(s,a) \;-\; \bar v_\mu(s) \Big)\\
    &= \frac{1}{1-\gamma} \sum_{s \in \gS} \left( \kappa \cdot (1-\eps) \bar{\ell}(\eta)^H - \kappa \cdot \bar{\ell}(\eta)^{H+1}\right)\\
    &= \frac{2\eps \eta \kappa (H+2)}{1-\gamma} \bar{\ell}(\eta)^H
\end{align*}
Now, $\inf_{\mu \in \Delta(\Pi)} \Gloc(\mu) \le \inf_{\mu \in \Delta(\Pi_0)} \Gloc(\mu) \le \inf_{\eta \rightarrow 0} \Gloc(\mu_\eta) = 0$.

% \textbf{Global Gap}. As $M(\pi^+_\eta) = M^B$ and $M(\pi^-_\eta) = M^A$ we have,
% \begin{align*}
%     \E_{\pi \sim \mu_\eta}[V_{M(\pi)}(\pi)] &= \frac{1}{2} V_{M^B}(\pi^+_\eta) + \frac{1}{2} V_{M^A}(\pi^-_\eta)\\
%     &= \frac{1}{2} \kappa \cdot F_B(1/2 + \eta) + \frac{1}{2} \kappa \cdot F_A(1/2-\eta) \\
%     &= \kappa \cdot \bar{\ell}(\eta)^H.
% \end{align*}
% Now for a comparator policy $\pi'$ with $y'_h = \pi'(1|s_h)$ we have, $\E_{\pi \sim \mu_\eta}[V_{M(\pi)}(\pi')] = \frac{1}{2} \kappa \left(F_A(y') + F_B(y') \right)$. Therefore,
% \begin{align*}
%     \Gglob(\mu_\eta) &= \max_{\pi' \in \Pi} \E_{\pi \sim \mu_\eta}[V_{M(\pi)}(\pi')] - \E_{\pi \sim \mu_\eta}[V_{M(\pi)}(\pi)] \\
%     &= \max_{y': y'_h \in [0,1]} \frac{\kappa}{2} \left( \prod_{h=1}^H (1-2\eps(1-y'_h)) + \prod_{h=1}^H (1-2\eps y'_h) \right) - \kappa \cdot \bar{\ell}(\eta)^H
% \end{align*}
% The above optimization problem is decoupled across the choices of $y'_h$, and choosing each $y'_h$ to be either $0$ or $1$ yields the optimal value. In particular, choosing $y'_h = 1$ for all $h \in [H]$ gives us
% \begin{align*}
% \Gglob(\mu_\eta) &\ge \kappa \left( \frac{1}{2} (1-2\eps)^H - (1-\eps -2\eps \eta)^H + \frac{1}{2}\right)\\
% &\ge \kappa \left( \frac{1 + (1-2\eps)^H}{2} - (1-\eps)^H\right)\quad \text{[As $\eta \le 1/2$]}
% \end{align*}
% Now using \Cref{lem:lbd-feps} we obtain $\Gglob(\mu_\eta) \ge \frac{\kappa}{16} \min \set{1,\eps^2H^2}$.
\end{proof}

\begin{proposition}\label{prop:cce}
Given a mixture $\mu$ of policies, let us define, for each $s \in \gS$,
\[
  \bar q_\mu(s,a) := \E_{\pi\sim\mu}\big[d^{\pi}_{P_\pi}(s)\,Q^{\pi}_{M(\pi)}(s,a)\big],
  \qquad
  \bar v_\mu(s) := \E_{\pi\sim\mu}\big[d^{\pi}_{P_\pi}(s)\,\langle \pi(\cdot\mid s), Q^{\pi}_{M(\pi)}(s,\cdot)\rangle\big].
\]
Then
\[
  \Gloc(\mu) \;=\; \frac{1}{1-\gamma}\sum_{s\in\gS}\Big( \max_{a \in \gA} \bar q_\mu(s,a) \;-\; \bar v_\mu(s) \Big).
\]
\end{proposition}

\begin{proof}
Expanding the inner product in $\Delta_\pi(s,\pi')$ and exchanging the (finite) sum over
$s$ with the expectation,
\begin{align*}
  \E_{\pi\sim\mu}\sum_s d^{\pi}_{P_\pi}(s)\Delta_\pi(s,\pi')
  &= \E_{\pi \sim \mu} \sum_s d^\pi_{P_\pi}(s) \left \langle \pi'(\cdot|s) - \pi(\cdot|s), Q^\pi_{M(\pi)}(s,\cdot) \right \rangle\\
  &=\sum_s \Big( \big\langle \pi'(\cdot\mid s),\, \bar q_\mu(s,\cdot)\big\rangle - \bar v_\mu(s)\Big).
\end{align*}
The summand depends on $\pi'$ only through $\pi'(\cdot\mid s)$, and the constraint set
$\Pi = \Delta(\gA)^{\gS}$ is a product over $s$. Hence the maximisation decouples across
states, and for each $s$ the maximum of the linear functional
$q \mapsto \langle q, \bar q_\mu(s,\cdot)\rangle$ over $q \in \Delta(\gA)$ is attained at a
vertex, giving $\max_a \bar q_\mu(s,a)$.
\end{proof}

\subsection{Global Stability of the Survival Chain Instance\label{app:global-stability-survival-chain}}

\begin{proof}[Proof of \Cref{lem:positive-global}]
    We will use \Cref{lem:cert} to provide a lower bound on the approximation of global mixed performative stability $\Gglob(\mu)$. Let $\pi^A$ (resp. $\pi^B$) plays action $1$ (resp. $2$) at every $s_h$. Then we define $\nu = \frac{1}{2} \delta_{\pi^A} + \frac{1}{2} \delta_{\pi^B}$. Now our aim is to establish a lower bound on the following quantity.
    $$
    \inf_{\pi \in \Pi} \left(\E_{\pi' \sim \nu}[V_{M(\pi)}(\pi')] - V_{M(\pi)}(\pi) \right).
    $$
    First suppose, we choose $\pi$ (with $y_h = \pi(1|s_h)$) such that $M(\pi) = M^A$ i.e. $F_A(y) \le F_B(y)$. Then,
    \begin{align*}
        \E_{\pi' \sim \nu} \left[ V_{M^A}(\pi')\right] = \frac{\kappa}{2} \left( 1 + (1-2\eps)^H\right) .
    \end{align*}
    Similarly, it can be verified that
        \begin{align*}
        \E_{\pi' \sim \nu} \left[ V_{M^B}(\pi')\right] = \frac{\kappa}{2} \left( 1 + (1-2\eps)^H\right) .
    \end{align*}
    On the other hand, for any policy $\pi \in \Pi$,
    \begin{align*}
    V_{M(\pi)}(\pi) &= \kappa \cdot \min \set{F_A(y), F_B(y)} \le \kappa \sqrt{F_A(y) F_B(y)}\\
    &= \kappa \sqrt{\prod_{h=1}^H (1-2\eps(1-y_h)) (1-2\eps y_h) }\\
    &\le \kappa (1-\eps)^H
    \end{align*}
    This is because each term $g(y) = (1-2\eps(1-y))(1-2\eps y)$ is maximized at $1/2$. Therefore, for any policy $\pi \in \Pi$
    \begin{align*}
    \E_{\pi' \sim \nu} \left[V_{M(\pi)}(\pi')\right] - V_{M(\pi)}(\pi) &\ge \kappa \left( \frac{1}{2} \left(1 + (1-2\eps)^H \right) - (1-\eps)^H\right)\\
    &\ge \kappa \frac{1}{16} \min \set{1, \eps^2 H^2}.
    \end{align*}
\end{proof}

\begin{lemma}\label{lem:lbd-feps}
    Let us define $f(\eps,H) = \frac{1 + (1-2\eps)^H}{2} - (1-\eps)^H$. Then,
    $$
    f(\eps,H) \ge \frac{1}{16} \min \set{1,\eps^2H^2}\ \text{ for } \eps \in [0,1/2], H \ge 2.
    $$
\end{lemma}
\begin{proof}
    Let $X_1,\ldots,X_H$ be i.i.d. Bernoulli trials with $\Pr[X_i = 1] = \eps$. First note that,
    $$
    1 = (1-\eps + \eps)^H = \sum_{k=0}^H {H \choose k} \eps^k (1-\eps)^k
    $$
    and
    $$
    (1-2\eps)^H = \sum_{k=0}^H {H \choose k} (-1)^k \eps^k (1-\eps)^{H-k}.
    $$
    Summing the above two identities and observing that the terms corresponding to $k$ odd cancel out, we obtain the following identity.
    $$
    \frac{1 + (1-2\eps)^H}{2} = \sum_{k: k \text{ even}}^H {H \choose k} \eps^k (1-\eps)^{H-k}
    $$
    Therefore, 
    $$
    f(\eps,H) = \frac{1 + (1-2\eps)^H}{2} - (1-\eps)^H = \sum_{\stackrel{k: k \text{ even}}{k\ge 2}}^H {H \choose k} \eps^k (1-\eps)^{H-k} \ge {H \choose 2} \eps^2 (1-\eps)^{H-2}.
    $$
    Now we consider two cases. First, if $(1-\eps)^H \ge 1/4$, we have
    $$
    f(\eps,H) \ge \frac{H^2}{4} \eps^2 \frac{(1-\eps)^H}{(1-\eps)^2} \ge \frac{1}{16} \eps^2 H^2.
    $$
    Second, if $(1-\eps)^H \le 1/4$, $f(\eps,H) \ge 1/4 + (1-2\eps)^H \ge 1/4$. Combining the two cases, we have the desired result.
\end{proof}

\begin{lemma}[Certificate lemma]\label{lem:cert}
Let $\nu \in \Delta(\Pi)$ and $c \in \R$ satisfy
\begin{equation}\label{eq:certcond}
  \E_{\pi'\sim\nu}\big[V_{M(\pi)}(\pi')\big] - V_{M(\pi)}(\pi) \;\ge\; c
  \qquad\text{for every } \pi \in \Pi .
\end{equation}
Then $\Gglob(\mu) \ge c$ for every $\mu \in \Delta(\Pi)$.
\end{lemma}

\begin{proof}
Fix $\mu$. A maximum dominates an average, so
\[
  \max_{\pi'\in\Pi}\E_{\pi\sim\mu}\big[V_{M(\pi)}(\pi')\big]
  \;\ge\; \E_{\pi'\sim\nu}\,\E_{\pi\sim\mu}\big[V_{M(\pi)}(\pi')\big]
  \;=\; \E_{\pi\sim\mu}\,\E_{\pi'\sim\nu}\big[V_{M(\pi)}(\pi')\big],
\]
by Fubini (all quantities are bounded by $1/(1-\gamma)$). Subtracting
$\E_{\pi\sim\mu}[V_{M(\pi)}(\pi)]$ and applying \eqref{eq:certcond} pointwise in $\pi$ gives
$\Gglob(\mu)\ge c$.
\end{proof}

\subsection{Local Stability without Any Assumption}

\begin{proof}[Proof of \Cref{thm:A}]
Fix an arbitrary comparator $\pi' \in \Pi$ and a state $s \in \gS$. Apply
Lemma~\ref{lem:hedge} at state $s$ with $p_t = \pi_t(\cdot\mid s)$, gain vectors
$g_t := w_t(s)\,Q_t(s,\cdot)$ and comparator $q := \pi'(\cdot\mid s)$. Since
$Q_t(s,a) \in [0,\tfrac1{1-\gamma}]$ and $w_t(s) \ge 0$, the gains lie in $[0,G_t]$ with
$G_t = w_t(s)/(1-\gamma)$. The update \eqref{eq:algA} is exactly the Hedge update on these
gains, so
\begin{equation}\label{eq:perstate}
  \sum_{t=1}^{T} w_t(s)\,\Delta_t(s,\pi')
  \;=\; \sum_{t=1}^{T}\big\langle \pi'(\cdot\mid s) - \pi_t(\cdot\mid s),\ g_t (s,\cdot)\big\rangle
  \;\le\; \frac{\ln A}{\eta} + \frac{\eta}{8(1-\gamma)^2}\sum_{t=1}^T w_t(s)^2 .
\end{equation}
Crucially, \eqref{eq:perstate} holds for every $s$ simultaneously and deterministically on
the realised trajectory, whatever be the sequence $(w_t, Q_t)$. % --- which is where the absence of
%any hypothesis on $M(\cdot)$ enters: the per-state learner treats $w_t$ and $Q_t$ as opaque
%adversarial data.

Summing \eqref{eq:perstate} over $s \in \gS$,
\[
  \sum_{t=1}^T \sum_{s} w_t(s)\Delta_t(s,\pi')
  \;\le\; \frac{S \ln A}{\eta} + \frac{\eta}{8(1-\gamma)^2}\sum_{t=1}^T \sum_s w_t(s)^2 .
\]
Since $w_t(\cdot)$ is a probability distribution over $\gS$,
$\sum_s w_t(s)^2 \le \big(\max_s w_t(s)\big)\sum_s w_t(s) \le 1$, so the second sum is at
most $T$, giving the cumulative bound
\begin{equation}\label{eq:cumA}
  \sum_{t=1}^T \sum_{s} w_t(s)\Delta_t(s,\pi')
  \;\le\; \frac{S \ln A}{\eta} + \frac{\eta\,T}{8(1-\gamma)^2}.
\end{equation}
Dividing \eqref{eq:cumA} by $T$,
\[
  \frac{1}{T}\sum_{t=1}^T\sum_s w_t(s)\Delta_t(s,\pi')
  \;\le\; \frac{S\ln A}{\eta T} + \frac{\eta}{8(1-\gamma)^2}.
\]
By construction of $\hat\mu$ as the uniform distribution over the realised iterates, the
left-hand side equals
\[
  \E_{\pi\sim\hat\mu}\Big[\textstyle\sum_s d^{\pi}_{P_\pi}(s)\Delta_\pi(s,\pi')\Big].
\]
The bound holds for every $\pi' \in \Pi$ and its right-hand side does not depend on $\pi'$, so it
holds for the maximum over $\pi'$. Multiplying by $\tfrac1{1-\gamma}$ gives the first result.

For the second, substitute $\eta = 2(1-\gamma)\sqrt{2S\ln A/T}$:
\[
  \frac{S\ln A}{\eta T} = \frac{1}{2\sqrt2}\cdot\frac{\sqrt{S\ln A}}{(1-\gamma)\sqrt{T}},
  \qquad
  \frac{\eta}{8(1-\gamma)^2} = \frac{\sqrt 2}{4}\cdot\frac{\sqrt{S\ln A}}{(1-\gamma)\sqrt{T}},
\]
and $\tfrac{1}{2\sqrt2} + \tfrac{\sqrt2}{4} = \tfrac{1}{\sqrt 2}$, giving
$\Gloc(\hat\mu) \le \tfrac{1}{(1-\gamma)^2}\sqrt{S\ln A/(2T)}$.
\end{proof}

\begin{lemma}[Hedge with time-varying gain ranges]\label{lem:hedge}
Fix a state $s$. Let $p_1 = \mathrm{Unif}(\gA)$ and, for gain vectors
$g_t \in [0, G_t]^{\gA}$, let $p_{t+1}(a) \propto p_t(a)\exp(\eta\, g_t(a))$ with $\eta > 0$.
Then for every $q \in \Delta(\gA)$,
\[
  \sum_{t=1}^{T}\big\langle q - p_t,\ g_t\big\rangle \;\le\; \frac{\ln A}{\eta}
  \;+\; \frac{\eta}{8}\sum_{t=1}^{T} G_t^2 .
\]
\end{lemma}

\begin{proof}
Let us define $Z_t := \sum_a p_1(a)\exp\big(\eta\sum_{\tau<t} g_\tau(a)\big)$, so $Z_1 = 1$ and
$p_t(a) = p_1(a)\exp(\eta\sum_{\tau<t}g_\tau(a))/Z_t$. Then
\[
  \ln\frac{Z_{t+1}}{Z_t} = \ln \E_{a \sim p_t}\big[e^{\eta g_t(a)}\big].
\]
The random variable $\eta g_t(a)$, $a \sim p_t$, takes values in an interval of length at
most $\eta G_t$, so Hoeffding's lemma gives
$\ln\E[e^{\eta g_t}] \le \eta\langle p_t, g_t\rangle + \eta^2 G_t^2/8$. Summing over
$t = 1,\dots,T$ and telescoping,
\[
  \ln Z_{T+1} = \sum_{t=1}^T \ln \frac{Z_{t+1}}{Z_t} = \sum_{t=1}^T \ln \E_{a \sim p_t}\left[e^{\eta g_t(a)} \right]\;\le\; \eta \sum_{t=1}^T \langle p_t, g_t\rangle + \frac{\eta^2}{8}\sum_{t=1}^T G_t^2 .
\]
For the lower bound, for any $q \in \Delta(\gA)$, we can apply Jensen's inequality to the
convex function $-\ln$.
\[
  \ln Z_{T+1}
  = \ln \sum_a p_1(a) e^{\eta \sum_t g_t(a)}
  \ \ge\ \sum_a q(a)\ln\frac{p_1(a) e^{\eta\sum_t g_t(a)}}{q(a)}
  = \eta \sum_t \langle q, g_t\rangle - \mathrm{KL}(q \,\|\, p_1),
\]
and $\mathrm{KL}(q\|p_1) \le \ln A$ since $p_1$ is uniform. Combining the above two displayed equations and
dividing by $\eta$ gives the claim.
\end{proof}

\subsection{Finite Sample Analysis of Local Stability}

\begin{proof}[Proof of \Cref{thm:E}]
Fix $\pi' \in \Pi$ and let us write $\xi_t(s,a) := g_t(s,a) - \hat g_t(s,a)$. Then we can establish the following decomposition
\begin{align}\label{eq:decompE}
  \sum_{t}\sum_s w_t(s)\Delta_t(s,\pi')
  &= \sum_t \sum_s \left \langle \pi'(\cdot|s) - \pi_t(\cdot|s), \hat{g}_t(s,\cdot)\right \rangle \nonumber \\
  &= \underbrace{\sum_{t}\sum_s \big\langle \pi'(\cdot| s)-\pi_t(\cdot| s),\ \hat g_t(s,\cdot)\big\rangle}_{(\mathrm{I})}
  + \underbrace{\sum_{t}\sum_s \big\langle \pi'(\cdot| s)-\pi_t(\cdot|s),\ \xi_t(s,\cdot)\big\rangle}_{(\mathrm{II})} .
\end{align}

We now bound the two terms separately. Let us also write $\gF_t$ to denote the $\sigma$-algebra generated by $M_t, \pi_t$ or equivalently, the trajectories drawn at iteration $t$. 

\textbf{Term (I).} The empirical distribution satisfies $\sum_s \widehat{w}_t(s)^2 \le 1$. Moreover, the estimate of $Q_t(s,a)$ is bounded between $0$ and $1/(1-\gamma)$ as it is $1/(1-\gamma)$ times the reward at the $\tau$-th step. Therefore, the proof of Theorem~\ref{thm:A} applies verbatim with $w_t$ replaced by $\widehat w_t$ giving
$$(\mathrm{I}) \le \frac{S\ln A}{\eta} + \frac{\eta T}{8(1-\gamma)^2}$$ as in \eqref{eq:cumA}.

\textbf{Term (II).} 
 Let us write $\Xi(s,a) := \sum_t \xi_t(s,a)$ and
$Y_T := \sum_t \sum_s \langle \pi_t(\cdot\mid s), \xi_t(s,\cdot)\rangle$. Since the constraint
set is the product $\Delta(\gA)^{\gS}$,
\begin{align*}
\max_{\pi'\in\Pi}\,(\mathrm{II}) &= \max_{\pi'\in\Pi} \sum_s \sum_t \left \langle \pi'(\cdot|s) - \pi_t(\cdot|s),\ \xi_t(s,\cdot)\right \rangle\\
&= \max_{\pi' \in \Pi} \sum_s \sum_a \pi'(a|s) \Xi(s,a) - Y_T\\
&= \sum_s \max_{a} \Xi(s,a) - Y_T.
\end{align*}
We now bound the two terms separately.

\textbf{Bounding $\sum_s \max_a \Xi(s,a)$}: Using \Cref{lem:bound-xi-tsa} we have the following inequality for any $\lambda \in [0,(1-\gamma)]$.
\begin{align*}
\Pr\left( \forall s,a, \Xi(s,a) \le (e-2)\lambda \left(\frac{2}{k(1-\gamma)^2} + \frac{1}{n(1-\gamma)^2} \right) \sum_t w_t(s) + \frac{\ln(SA/\delta)}{\lambda}\right) \ge 1-\delta.
\end{align*}
This gives us,
\begin{align*}
\sum_s \max_a \Xi(s,a) &\le (e-2)\lambda \left(\frac{2}{k(1-\gamma)^2} + \frac{1}{n(1-\gamma)^2} \right) \sum_t \sum_s w_t(s) + \frac{S\ln(SA/\delta)}{\lambda}\\
&\le (e-2)\lambda \left(\frac{2}{k(1-\gamma)^2} + \frac{1}{n(1-\gamma)^2} \right) T + \frac{S\ln(SA/\delta)}{\lambda}
\end{align*}
Substituting $\lambda = (1-\gamma) \sqrt{\frac{S \ln(SA/\delta)}{(e-2)T(2/k + 1/n)}}$ we obtain the following upper bound.
\begin{equation}\label{eq:bound-II-part-a}
\sum_s \max_a \Xi(s,a) \le \frac{2}{1-\gamma} \sqrt{(e-2)\left(\frac{2}{k} + \frac{1}{n} \right) T S \ln(SA/\delta)}
\end{equation}
Note that, as long as $T \ge \frac{S \ln(SA/\delta)}{(e-2)(2/k + 1/n)}$ we are guaranteed that $\lambda \in [0,(1-\gamma)]$ which is required to apply Freedman's inequality.

\textbf{Bounding $-Y_T$}: Let us define $y_{t,s} =  - \left \langle \pi_t(\cdot|s), \xi_t(s,\cdot)\right \rangle$. Then $-Y_t = \sum_s \sum_t y_{t,s}$. We can apply \Cref{lem:bound-yts} to obtain the following inequality for any $\lambda \in [0,(1-\gamma)]$.
\begin{align*}
\Pr\left( \forall s, \sum_t y_{t,s} \le (e-2)\lambda \left(\frac{2}{k(1-\gamma)^2} + \frac{1}{n(1-\gamma)^2} \right) \sum_t w_t(s) + \frac{\ln(S/\delta)}{\lambda}\right) \ge 1-\delta.
\end{align*}

This gives us,
\begin{align*}
    -Y_T &= \sum_s \sum_t y_{t,s} \le (e-2)\lambda \left(\frac{2}{k(1-\gamma)^2} + \frac{1}{n(1-\gamma)^2} \right) \sum_t \sum_s w_t(s) + \frac{S\ln(S/\delta)}{\lambda}\\
    &\le \frac{(e-2)\lambda}{(1-\gamma)^2} \left( \frac{2}{k} + \frac{1}{n}\right) T + \frac{S \ln(S/\delta)}{\lambda}
\end{align*}
Substituting $\lambda = (1-\gamma) \sqrt{\frac{S \ln(SA/\delta)}{(e-2)T(2/k + 1/n)}}$ we obtain the following upper bound.
\begin{equation}\label{eq:bound-II-part-b}
    -Y_T \le \frac{2}{1-\gamma} \sqrt{(e-2)\left(\frac{2}{k} + \frac{1}{n} \right) T S \ln(S/\delta)}
\end{equation}
Now using \Cref{eq:bound-II-part-a} and \Cref{eq:bound-II-part-b} we obtain the following upper bound on part (II).
$$
\max_{\pi'} (II) \le \frac{4}{1-\gamma} \sqrt{(e-2)\left(\frac{2}{k} + \frac{1}{n} \right) T S \ln(SA/\delta)}
$$
Recall the definition of $\Gloc(\widehat{\mu})$.
\begin{align*}
    \Gloc(\widehat{\mu}) &= \frac{1}{(1-\gamma) T} \sum_t \sum_s w_t(s) \Delta_t(s,\pi') \\
    &\le \frac{S \ln A}{\eta T} + \frac{\eta}{8(1-\gamma)^2} + \frac{4}{(1-\gamma)^2} \sqrt{\frac{(2/k + 1/n) S \ln(SA/\delta)}{T} }
\end{align*}
The bound on the sample complexity follows from the observation that  iteration $t$ collects $n$ occupancy samples and $kA\abs{\text{supp}(\widehat{w}_t)} \le nkA$ value samples.
\end{proof}

\begin{lemma}[Estimator properties]\label{lem:est}
Conditioned on $\pi_t$ and $M_t$, an occupancy
sample is distributed as $d^{\pi_t}_{P_t}$, and a value sample at $(s,a)$ is an unbiased
estimate of $Q_t(s,a)$ taking values in $[0,\tfrac{1}{1-\gamma}]$.
\end{lemma}

\begin{proof}
For the first, $\Pr[s_\tau = s] = \sum_j (1-\gamma)\gamma^j\Pr[s_j = s] = d^{\pi_t}_{P_t}(s)$
by definition of the normalised occupancy. 

For the second,
$\E[\tfrac{r_\tau}{1-\gamma}] = \tfrac{1}{1-\gamma}\sum_j(1-\gamma)\gamma^j\,\E[r_j \mid s_0=s,a_0=a]
= Q^{\pi_t}_{M_t}(s,a)$, and $r_\tau \in [0,1]$ bounds the range.
\end{proof}

\begin{lemma}[Freedman's inequality, \cite{BLLR+11}]\label{lem:freedman}
    Let $\{X_t\}_{t=1}^T$ be a martingale difference sequence measurable with respect to the sequence of $\sigma$-algebras satisfying $X_t \le R$ a.s. for all $t \in [T]$. Then for any $\lambda \in (0,1/R]$ and $\delta \in (0,1)$
    $$
    \Pr \left( \sum_{t=1}^T X_t \le (e-2) \lambda \sum_{t=1}^T \E[X_t^2 | \gF_{t-1}] + \frac{\ln(1/\delta)}{\lambda}\right) \ge 1-\delta.
    $$
\end{lemma}

\begin{lemma}\label{lem:bound-xi-tsa}
  As in the proof of Theorem~\ref{thm:E}, let $\xi_t(s,a) = w_t(s)Q_t(s,a) - \widehat{w}_t(s) \widehat{Q}_t(s,a)$. Then $\xi_t(s,a)$ is a martingale difference sequence with respect to $\gF_{t-1}$ and for any $\lambda \in [0,(1-\gamma)]$
  \[
 \Pr\left( \forall s,a, \ \sum_t \xi_t(s,a) \le (e-2)\lambda \left(\frac{2}{k(1-\gamma)^2} + \frac{1}{n(1-\gamma)^2} \right) \sum_t w_t(s) + \frac{\ln(SA/\delta)}{\lambda}\right) \ge 1-\delta.
  \]
\end{lemma}
\begin{proof}
  First, we show that $\widehat{g}_t$ is an unbiased estimator for $g_t$. Let us write $b_t$ to denote the occupancy batch collected during iteration $t$. Conditioned on the batch $b_t$, the estimator $\widehat{w}_t$ and $\text{supp}(\widehat{w}_t)$ are determined. The value batch uses independent samples, and therefore by \Cref{lem:est},
$\E[\widehat Q_t(s,a)\mid \gF_{t-1}, b_t] = Q_t(s,a)$ for every
$s \in \mathrm{supp}(\hat w_t)$. Since conditioned on the occupancy batch $b_t$, the estimators $\widehat w_t$ and $\widehat{Q}_t$ are independent, we have 
\begin{align}
\begin{split}\label{eq:cond-indepence-estimators}
\E[\widehat{w}_t(s) \widehat{Q}_t(s,a) | \gF_{t-1}] &= \E \left[\E_{b_t}[\widehat{w}_t(s) \widehat{Q}_t(s,a) | \gF_{t-1}, b_t]\right] \\
&= \E\left[ \E_{b_t}[\widehat{w}_t(s) | \gF_{t-1},b_t] \cdot \E_{b_t}[\widehat{Q}_t(s,a) | \gF_{t-1},b_t] \right]\\
&= \E[\widehat{w}_t(s) | \gF_{t-1}] Q_t(s,a)
\end{split}
\end{align}
for all $s \in \text{supp}(\widehat{w}_t)$. The claim also holds trivially for $s \notin \text{supp}(\widehat{w}_t)$. Now taking expectation over $b_t$ and using the observation $\E[\widehat{w}_t(s) | \gF_{t-1}] = w_t(s)$ we obtain
$$
\E[\widehat{g}_t(s,a) | \gF_{t-1}] = \E[\widehat{w}_t(s) \widehat{Q}_t(s,a) | \gF_{t-1}] = w_t(s) Q_t(s,a) = g_t(s,a)
$$
Therefore, $\E[\xi_{t}(s,a) | \gF_{t-1}] = \E[\widehat{g}_t(s,a) - g_t(s,a) | \gF_{t-1}] = 0$ and $\{\xi_{t}(s,a)\}_{t=1}^T$ is a Martingale difference sequence. We now bound $\sum_t \xi_t(s,a)$ using Freedman's inequality~\Cref{lem:freedman}. 
\begin{align*}
    &\E[\xi^2_t(s,a)|\gF_{t-1}] = \text{Var}(\widehat{g}_t(s,a) | \gF_{t-1}) =  \text{Var}(\widehat{w}_t(s,a) \widehat{Q}_t(s,a) | \gF_{t-1})\\
    = &\E\left[ \widehat{w}_t(s,a)^2 | \gF_{t-1}\right] \E\left[ \widehat{Q}_t(s,a)^2 | \gF_{t-1}\right] - \left( \E\left[ \widehat{w}_t(s,a) | \gF_{t-1}\right] \right)^2 \left( \E\left[ \widehat{Q}_t(s,a)^2 | \gF_{t-1}\right] \right)^2
\end{align*}
The last line follows the same argument as displayed in \Cref{eq:cond-indepence-estimators}. Now observe, that $\widehat{w}_t(s)$ is the average of $n$ Bernoulli trials with success probability $w_t(s)$. Therefore, $\E[\widehat{w}_t(s)^2] = w_t(s) (1-w_t(s))/n + w_t^2(s)$. This gives us the following upper bound.
\begin{align*}
    \E[\xi^2_t(s,a)|\gF_{t-1}] &= \left( \frac{w_t(s) (1-w_t(s))}{n} + w_t(s)^2\right) \left( Q_t(s,a)^2 + \text{Var}\left( \widehat{Q}_t(s,a) | \gF_{t-1}\right)\right) - w_t(s)^2 Q_t(s,a)^2\\
    &= \left( \frac{w_t(s) (1-w_t(s))}{n} + w_t(s)^2 \right)\cdot \text{Var}\left( \widehat{Q}_t(s,a) | \gF_{t-1}\right) + \frac{w_t(s) (1-w_t(s))}{n} \cdot Q_t(s,a)^2
\end{align*}
We now use $w_t(s) \in [0,1]$, and $Q_t(s,a) \in [0,1/(1-\gamma)]$. Moreover, $\widehat{Q}_t(s,a)$ is the sample average of $k$ independent random variables bounded between $0$ and $1/(1-\gamma)$, giving us $\text{Var}\left( \widehat{Q}_t(s,a) | \gF_{t-1}\right) \le 1/k \cdot 1/(1-\gamma)^2$.  
\begin{align}
\label{eq:var-upper-bound}
\begin{split}
    \E[\xi^2_t(s,a)|\gF_{t-1}] &\le \left( \frac{w_t(s)}{n} + w_t(s)^2 \right) \cdot \frac{1}{k(1-\gamma)^2} + \frac{w_t(s)}{n} \cdot \frac{1}{(1-\gamma)^2}\\
    &\le \frac{2w_t(s)}{k} \cdot \frac{1}{(1-\gamma)^2} + \frac{w_t(s)}{n} \cdot \frac{1}{(1-\gamma)^2}.
    \end{split}
\end{align}
We can apply \Cref{lem:freedman} and union bound over the $SA$ tuples to obtain the following inequality for any $\lambda \in [0,(1-\gamma)]$.
\begin{align} \label{eq:freedman-bound}
\Pr\left( \forall s,a \sum_t \xi_t(s,a) \le (e-2)\lambda \left(\frac{2}{k(1-\gamma)^2} + \frac{1}{n(1-\gamma)^2} \right) \sum_t w_t(s) + \frac{\ln(SA/\delta)}{\lambda}\right) \ge 1-\delta.
\end{align}
\end{proof}

\begin{lemma}\label{lem:bound-yts}
  As in the proof of Theorem~\ref{thm:E}, let $y_{t,s} = - \langle \pi_t(\cdot|s), \xi_t(s,\cdot)\rangle$ with $\xi_t(s,a) = w_t(s)Q_t(s,a) - \widehat{w}_t(s) \widehat{Q}_t(s,a)$. Then $y_{t,s}$ is a martingale difference sequence with respect to $\gF_{t-1}$ and for any $\lambda \in [0,(1-\gamma)]$
  \[
 \Pr\left( \forall s, \ \sum_t y_{t,s} \le (e-2)\lambda \left(\frac{2}{k(1-\gamma)^2} + \frac{1}{n(1-\gamma)^2} \right) \sum_t w_t(s) + \frac{\ln(S/\delta)}{\lambda}\right) \ge 1-\delta.
  \]
\end{lemma}
\begin{proof}
  First note that,
  \begin{align*}
\E[y_{t,s}|\gF_{t-1}] &= - \sum_a \E\left[ \pi_t(a|s) \xi_t(s,a) | \gF_{t-1}\right]\\
&= -\sum_{a} \E_{\pi_t} \left[ \pi_t(a|s) \E[\xi_t(s,a) | \gF_{t-1}, \pi_t]\right] = 0.
\end{align*}
The last line follows since conditioned on the policy $\pi_t$, the mean of $\widehat{g}_t(s,a) = \widehat{w}_t(s) \widehat{Q}_t(s,a)$ equals $g_t(s,a) = w_t(s)Q_t(s,a)$ and $\E[\xi_t(s,a) | \gF_{t-1}, \pi_t] = 0$. Therefore, $y_{t,s}$ is a martingale difference sequence. Moreover, from the definition of $\xi_t(s,a)$ we have, $\abs{y_{t,s}} \le \norm{\pi_t(\cdot|s)}_1 \max_a \xi_t(s,a) \le 1/(1-\gamma)$. 
\begin{align*}
    \E[y_{t,s}^2 | \gF_{t-1}] &= \E\left[ \left( \sum_a \pi_t(a|s) \xi_t(s,a) \right)^2 \bigg | \gF_{t-1}\right]\\
    &= \E\left[ \sum_a \pi_t(a|s)^2 \xi_t(s,a)^2 \bigg | \gF_{t-1} \right] + \E\left[ \sum_{a,b} \pi_t(a|s) \pi_t(b|s) \xi_t(s,a) \xi_t(s,b) \bigg | \gF_{t-1} \right]\\
    &\le \sum_a \pi_t(a|s) \E[\xi_t(s,a)^2 | \gF_{t-1}]
\end{align*}
The last line follows from the observation that conditioned on $\gF_{t-1}$ and occupancy batch samples, $\xi_t(s,a)$ and $\xi_t(s,b)$ as they are estimated from different sets of trajectories. Therefore,
$$\E[\xi_t(s,a) \xi_t(s,b) | \gF_{t-1},b_t] = \E[\xi_t(s,a) | \gF_{t-1},b_t] \E[ \xi_t(s,b) | \gF_{t-1},b_t] = 0.$$
Now using \Cref{eq:var-upper-bound} we obtain,
$$
\E[y_{t,s}^2 | \gF_{t-1}]  \le \frac{2w_t(s)}{k} \cdot \frac{1}{(1-\gamma)^2} + \frac{w_t(s)}{n} \cdot \frac{1}{(1-\gamma)^2}
$$
We can now apply \Cref{lem:freedman} and union bound over $S$ states to obtain the following inequality for any $\lambda  \in [0,(1-\gamma)]$.
$$
\Pr\left( \forall s, \ \sum_t y_{t,s} \le (e-2)\lambda \left(\frac{2}{k(1-\gamma)^2} + \frac{1}{n(1-\gamma)^2} \right) \sum_t w_t(s) + \frac{\ln(S/\delta)}{\lambda}\right) \ge 1-\delta.
$$
\end{proof}

\subsection{Global Stability Under Bounded Transition Range}

\begin{proof}[Proof of \Cref{thm:global-stability}]
Fix $\pi' \in \Pi$. Applying Lemma~\ref{lem:hedge} at each state $s$ with
$g_t = Q_t(s,\cdot) \in [0,\tfrac{1}{1-\gamma}]^{\gA}$, i.e. $G_t \equiv \tfrac1{1-\gamma}$,
gives the \emph{uniform} per-state regret bound for every $s \in \gS$.
\begin{equation}\label{eq:RT}
  \sum_{t=1}^T \Delta_t(s,\pi') \;\le\; \frac{\ln A}{\eta} + \frac{\eta T}{8(1-\gamma)^2}
  \;=:\; R_T
\end{equation}
and with $\eta = 2(1-\gamma)\sqrt{2\ln A/T}$ one gets, exactly as at the end of the proof of
Theorem~\ref{thm:A}, $R_T = \tfrac{1}{1-\gamma}\sqrt{T\ln A/2}$.

By Lemma~\ref{lem:pdl} applied in $M_t = (P_t, r_t)$,
\begin{equation}\label{eq:intermediate-global-mixed}
  \frac1T\sum_{t=1}^T \big[V_{M_t}(\pi') - V_{M_t}(\pi_t)\big]
  = \frac{1}{(1-\gamma)T}\sum_{t=1}^T \sum_s d^{\pi'}_{P_t}(s)\,\Delta_t(s,\pi').
\end{equation}
Write $d^{\pi'}_{P_t}(s) = d^{\pi'}_{\bar P}(s) + e_t(s)$. By Assumption~\ref{ass:range} and
Lemma~\ref{lem:sim},
\begin{equation}\label{eq:et}
  \|e_t\|_1 \le \frac{\gamma\eps_P}{1-\gamma} \qquad \text{for every } t .
\end{equation}
We now replace $d^{\pi'}_{P_t}(s)$ in \Cref{eq:intermediate-global-mixed} by $d^{\pi'}_{\bar P}(s) + e_t(s)$ and analyze the two terms separately.

\emph{Main term.} The weights $d^{\pi'}_{\bar P}(s)$ are nonnegative, and sum to one. Moreover, $\bar P$ is a single fixed kernel.
Therefore we may exchange the order of summation and apply \eqref{eq:RT} state-wise.
\[
  \sum_{t=1}^T \sum_s d^{\pi'}_{\bar P}(s)\Delta_t(s,\pi')
  = \sum_s d^{\pi'}_{\bar P}(s) \sum_{t=1}^T \Delta_t(s,\pi')
  \le \sum_s d^{\pi'}_{\bar P}(s)\, R_T
  = R_T .
\]

\emph{Perturbation term.} Since $Q_t(s,a) \in [0,1/(1-\gamma)]$, $\Delta_t(s,\pi') = \langle \pi'(\cdot|s) - \pi(\cdot|s), Q_t(s,\cdot)\rangle \in [0,1/(1-\gamma)]$ for any state $s$. By H\"older's inequality,  and \eqref{eq:et},
\[
  \Big|\sum_{t=1}^T\sum_s e_t(s)\Delta_t(s,\pi')\Big|
  \le \sum_{t=1}^T \|e_t\|_1 \max_s |\Delta_t(s,\pi')|
  \le T \cdot \frac{\gamma\eps_P}{1-\gamma}\cdot\frac{1}{1-\gamma} .
\]

\emph{Combining.} Dividing by $(1-\gamma)T$,
\[
  \frac1T\sum_{t=1}^T\big[V_{M_t}(\pi') - V_{M_t}(\pi_t)\big]
  \;\le\; \frac{R_T}{(1-\gamma)T} + \frac{\gamma\eps_P}{(1-\gamma)^3}
  \;=\; \frac{1}{(1-\gamma)^2}\sqrt{\frac{\ln A}{2T}} + \frac{\gamma\eps_P}{(1-\gamma)^3}.
\]
The left-hand side equals $\E_{\pi\sim\hat\mu}[V_{M(\pi)}(\pi')] - \E_{\pi\sim\hat\mu}[V_{M(\pi)}(\pi)]$,
the right-hand side is independent of $\pi'$, and $\pi'$ was arbitrary; taking the maximum
over $\pi'$ gives the claim.
\end{proof}

\begin{lemma}[Occupancy simulation lemma]\label{lem:sim}
Let $P, \bar P$ be transition kernels with $\max_{s,a}\|P(\cdot| s,a) - \bar P(\cdot| s,a)\|_1 \le \eps$.
Then for every $\pi' \in \Pi$,
\[
  \big\| d^{\pi'}_{P} - d^{\pi'}_{\bar P}\big\|_1 \;\le\; \frac{\gamma\,\eps}{1-\gamma},
\]
where the norm is over state--action pairs; the same bound holds for the state marginals.
\end{lemma}

\begin{proof}
For a kernel $P$ define the affine operator $T_P$ on $\R^{\gS\times\gA}$ by
\[
  (T_P d)(s,a) \;=\; \pi'(a\mid s)\Big[(1-\gamma)\rho(s) + \gamma\sum_{s',a'} P(s\mid s',a')\, d(s',a')\Big].
\]
The Bellman flow constraints state exactly that $d^{\pi'}_P$ is the unique fixed point of
$T_P$ in the simplex.

\emph{Contraction.} For $d, d'$ in the simplex,
\begin{align*}
  \|T_P d - T_P d'\|_1
  &= \gamma \sum_{s,a} \pi'(a\mid s)\Big| \sum_{s',a'} P(s\mid s',a')\,(d-d')(s',a')\Big| \\
  &\le \gamma \sum_{s} \sum_{s',a'} P(s\mid s',a')\,\big|(d-d')(s',a')\big|
  \;=\; \gamma\|d - d'\|_1,
\end{align*}
using $\sum_a \pi'(a\mid s) = 1$ and then exchanging the order of summation.

\emph{Perturbation.} For any $d$ in the simplex,
\begin{align*}
  \|T_P d - T_{\bar P} d\|_1
  &= \gamma\sum_{s}\Big|\sum_{s',a'}\big(P - \bar P\big)(s\mid s',a')\,d(s',a')\Big| \\
  &\le \gamma \sum_{s',a'} d(s',a')\,\big\|P(\cdot\mid s',a') - \bar P(\cdot\mid s',a')\big\|_1
  \;\le\; \gamma\eps .
\end{align*}

\emph{Conclusion.} Writing $d := d^{\pi'}_P$ and $\bar d := d^{\pi'}_{\bar P}$,
\[
  \|d - \bar d\|_1 = \|T_P d - T_{\bar P}\bar d\|_1
  \le \|T_P d - T_P \bar d\|_1 + \|T_P \bar d - T_{\bar P}\bar d\|_1
  \le \gamma \|d - \bar d\|_1 + \gamma\eps ,
\]
so $\|d - \bar d\|_1 \le \gamma\eps/(1-\gamma)$. The marginal bound follows since summing
over actions is a contraction in $\ell_1$.
\end{proof}

\subsection{Finite Sample Analysis of Global Stability}

\begin{proof}[Proof of \Cref{thm:finite-global}]
    %The proof proceeds similar to the proof of \Cref{thm:E}. 
    Let us write $\zeta_t(s,a) = Q_t(s,a) - \widehat{Q}_t(s,a)$. Then we can write
    \begin{align*}
      \Gglob(\hat\mu) &= \max_{\pi' \in \Pi} \frac{1}{T}\sum_{t=1}^T V_{M_t}(\pi') - V_{M_t}(\pi_t)\\
      &= \max_{\pi' \in \Pi} \frac{1}{(1-\gamma)T}\sum_{t=1}^T \sum_{s} d^{\pi'}_{P_t}(s)\Delta_t(s,\pi')\\
      &= \max_{\pi' \in \Pi} \frac{1}{(1-\gamma)T}\sum_{t=1}^T \sum_{s} d^{\pi'}_{P_t}(s)\left \langle \pi'(\cdot|s) - \pi_t(\cdot|s) , Q_t(s,\cdot)\right \rangle\\
      &= \max_{\pi' \in \Pi} \frac{1}{(1-\gamma)T}\underbrace{\sum_{t=1}^T \sum_{s} d^{\pi'}_{P_t}(s)\left \langle \pi'(\cdot|s) - \pi_t(\cdot|s) , \widehat{Q}_t(s,\cdot)\right \rangle}_{:= (I)} \\
      &+ \frac{1}{(1-\gamma)T}\underbrace{\sum_{t=1}^T \sum_{s} d^{\pi'}_{P_t}(s)\left \langle \pi'(\cdot|s) - \pi_t(\cdot|s) , \zeta_t(s,\cdot)\right \rangle}_{:= (II)}.
    \end{align*}
    We now proceed similar to the proof of \Cref{thm:E}, and bound $(I)$ using \Cref{lem:hedge} and $(II)$ using Freedman's inequality.

    \textbf{Bounding $(I)$:} Since $\widehat{Q}_t(s,a) \in [0,1/(1-\gamma)]$, we can apply \Cref{lem:hedge} to get that for any $\pi' \in \Pi$,
    \begin{align*}
        \sum_{t=1}^T \Delta_t(s,\pi') \le \frac{\ln A}{\eta} + \frac{\eta T}{8(1-\gamma)^2} =: R_T.
    \end{align*}
    Substituting $\eta = 2(1-\gamma)\sqrt{2\ln A/T}$, we get $R_T = \frac{1}{1-\gamma}\sqrt{T\ln A/2}$. Now we can proceed similar to the proof of \Cref{thm:global-stability} to get that
    \begin{equation}
      \label{eq:term-I-global-bound}
      \max_{\pi' \in \Pi} (I) \le \frac{1}{1-\gamma}\sqrt{T\ln A/2} + \frac{T \gamma \epsilon_P}{(1-\gamma)^2}.
    \end{equation}

    \textbf{Bounding $(II)$:} Writing $d^{\pi'}_{P_t}(s) = d^{\pi'}_{\bar P}(s) + e_t(s)$, we can write
    \begin{align*}
       \max_{\pi' \in \Pi} (II) &= \max_{\pi' \in \Pi} \sum_{t=1}^T \sum_{s} d^{\pi'}_{P_t}(s)\left \langle \pi'(\cdot|s) - \pi_t(\cdot|s) , \zeta_t(s,\cdot)\right \rangle\\
       &= \max_{\pi' \in \Pi} \Bigg[ \sum_{t=1}^T \sum_{s} d^{\pi'}_{\bar P}(s)\left \langle \pi'(\cdot|s) - \pi(\cdot|s) , \zeta_t(s,\cdot) \right \rangle\\
       &\qquad\qquad - \sum_{t=1}^T \sum_{s} e_t(s)\left \langle \pi'(\cdot|s) - \pi_t(\cdot|s), \zeta_t(s,\cdot) \right \rangle \Bigg]\\
    \end{align*}
    Using \Cref{ass:range} and \Cref{lem:sim}, we can bound the second term as follows:
    $$\frac{1}{1-\gamma}\sum_{t=1}^T \sum_s \abs{e_t(s)} \norm{\pi'(\cdot|s) - \pi_t(\cdot|s)}_1 \le \frac{T}{1-\gamma} \norm{e_t}_1 \le \frac{T \gamma \epsilon_P}{(1-\gamma)^2}.$$
    This gives us,
    \begin{align*}
        \max_{\pi \in \Pi} (II) &\le  \max_{\pi' \in \Pi} \sum_{t=1}^T \sum_{s} d^{\pi'}_{\bar P}(s)\left \langle \pi'(\cdot|s) - \pi_t(\cdot|s) , \zeta_t(s,\cdot) \right \rangle + \frac{T \gamma \epsilon_P}{(1-\gamma)^2}\\
        &\le \max_{\pi' \in \Pi} \sum_{s} d^{\pi'}_{\bar P}(s) \sum_{t=1}^T \left \langle \pi'(\cdot|s), \zeta_t(s,\cdot) \right \rangle + \max_{\pi' \in \Pi} \sum_{s} d^{\pi'}_{\bar P}(s) \sum_{t=1}^T \left \langle \pi_t(\cdot|s), -\zeta_t(s,\cdot) \right \rangle + \frac{T \gamma \epsilon_P}{(1-\gamma)^2}\\
        &\le \max_{\pi' \in \Pi} \sum_{s} d^{\pi'}_{\bar P}(s) \max_a \sum_{t=1}^T \zeta_t(s,a) + \max_{\pi' \in \Pi} \sum_{s} d^{\pi'}_{\bar P}(s) \sum_{t=1}^T \left \langle \pi_t(\cdot|s), -\zeta_t(s,\cdot) \right \rangle + \frac{T \gamma \epsilon_P}{(1-\gamma)^2}.\\
    \end{align*}
    Let us now write $Z_T(s,a) = \sum_{t=1}^T \zeta_t(s,a)$, and $y_{t,s} = \left \langle \pi_t(\cdot|s), \zeta_t(s,\cdot) \right \rangle$. Then we can write
    \begin{align}\label{eq:intermediate-bound-II}
      \max_{\pi' \in \Pi} (II) &\le \max_{\pi' \in \Pi} \sum_{s} d^{\pi'}_{\bar P}(s) \max_a Z_T(s,a) + \max_{\pi' \in \Pi} \sum_{s} d^{\pi'}_{\bar P}(s) \sum_{t=1}^T y_{t,s} + \frac{T \gamma \epsilon_P}{(1-\gamma)^2}
    \end{align}
    We now bound the remaining two terms. First, note that $\E[\zeta_t(s,a) | \gF_{t-1}] = 0$ as $\widehat{Q}_t(s,a)$ is an unbiased estimate of $Q_t(s,a)$. Moreover, $\widehat{Q}_t(s,a)$ is the average of $k$ independent random variables in $[0,1/(1-\gamma)]$, and hence we have
    $$
      \E[\zeta_t(s,a)^2 | \gF_{t-1}] =  \Var(\widehat{Q}_t(s,a)) \le \frac{1}{k(1-\gamma)^2}.
    $$
    Therefore, by applying Freedman's inequality~\Cref{lem:freedman}, we get that with probability at least $1-\delta$, for any $\lambda \in [0,(1-\gamma)]$
    $$
    \forall s,a,\ Z_T(s,a) =\sum_{t=1}^T \zeta_t(s,a) \le (e-2)\lambda \cdot \frac{T}{k(1-\gamma)^2} + \frac{\ln(SA/\delta)}{\lambda}.
    $$
    Substituting $\lambda = \sqrt{k(1-\gamma)^2 \ln(SA/\delta)/(T(e-2))}$, we get that with probability at least $1-\delta$, for all $s,a$,
    $$
    Z_T(s,a) \le \frac{2}{(1-\gamma)}\sqrt{\frac{T}{k} \ln(SA/\delta)}.
    $$
    Note that as long as $T \ge k \ln(SA/\delta)$, we can guarantee that the requried $\lambda \in [0,(1-\gamma)]$.
    
    Now we upper bound the second term in \Cref{eq:intermediate-bound-II}. Note that $y_{t,s}$ is a martingale difference sequence with respect to $\gF_t$,as $\E[y_{t,s}|\gF_{t-1}] = \E_{\pi_t} \left[ \E[y_{t,s}|\gF_{t-1},\pi_t] \right]= 0$ . Furthermore,  $|y_{t,s}| \le 1/(1-\gamma)$. 
    \begin{align*}
    \E[y_{t,s}^2 | \gF_{t-1}] &= \E\left[ \left( \sum_a \pi_t(a|s) \zeta_t(s,a) \right)^2 \bigg | \gF_{t-1}\right]\\
    &= \E\left[ \sum_a \pi_t(a|s)^2 \zeta_t(s,a)^2 \bigg | \gF_{t-1} \right] + \E\left[ \sum_{a,b} \pi_t(a|s) \pi_t(b|s) \zeta_t(s,a) \zeta_t(s,b) \bigg | \gF_{t-1} \right]\\
    &\le \sum_a \pi_t(a|s) \E[\zeta_t(s,a)^2 | \gF_{t-1}]\\
    &\le \frac{1}{k(1-\gamma)^2}
\end{align*}
    Therefore, we can apply Freedman's inequality again to get that with probability at least $1-\delta$, for any $\lambda \in [0,(1-\gamma)]$,
    $$
    \sum_{t=1}^T y_{t,s} \le (e-2)\lambda \cdot \frac{T}{k(1-\gamma)^2} + \frac{\ln(S/\delta)}{\lambda}.
    $$
    Substituting $\lambda = \sqrt{k(1-\gamma)^2 \ln(S/\delta)/(T(e-2))}$, we get that with probability at least $1-\delta$, for all $s$,
    $$
    \sum_{t=1}^T y_{t,s} \le \frac{2}{(1-\gamma)}\sqrt{\frac{T}{k} \ln(S/\delta)}.
    $$
    Substituting the above two bounds in \Cref{eq:intermediate-bound-II}, we get that with probability at least $1-2\delta$,
    \begin{align*}
      \max_{\pi' \in \Pi} (II) &\le \frac{2}{(1-\gamma)}\sqrt{\frac{T  \ln(SA/\delta)}{k}} + \frac{2}{(1-\gamma)}\sqrt{\frac{T}{k}  \ln(S/\delta)} + \frac{T \gamma \epsilon_P}{(1    -\gamma)^2} \\
      &\le \frac{4}{(1-\gamma)}\sqrt{\frac{T  \ln(SA/\delta)}{k}} + \frac{T \gamma \epsilon_P}{(1-\gamma)^2},
    \end{align*}
    Now adding the upper bound for $(I)$ given in \Cref{eq:term-I-global-bound}, and dividing by $T(1-\gamma)$ we get the desired bound.
 \end{proof}

 \subsection{Lower bound under Local Feedback}\label{sec:lower-bound}

 \begin{proof}[Proof of \Cref{thm:linear-lower-bound}]
  We will consider a slightly stronger interaction model than specified in \Cref{def:localfb}. During round $t$, the algorithm draws $m_t \ge 0$ trajectories. Trajectory
$(t,i)$ selects an action $A_{t,i} \in [K]$, and returns
$X_{t,i} \sim \mathrm{Bern}\big(p_{A_{t,i}}(\pi_t)\big)$, the indicator of absorption in $\top$.
The total budget satisfies $\sum_t m_t \le N$ almost surely, and the algorithm outputs
$\hat\mu := \mathrm{Unif}\{\pi_1,\dots,\pi_T\}$. 
%Write $N_j := \#\{(t,i) : A_{t,i} = j\}$, so
%that $\sum_{j\in[K]} N_j \le N$ pathwise.
Note that, letting the algorithm choose $A_{t,i}$, rather than sampling it from $\pi_t$, only
strengthens the oracle and hence the lower bound
%; the reader who prefers on-policy sampling may
%%substitute $A_{t,i}\sim\pi_t(\cdot\mid s_0)$ throughout, which changes nothing below. Note also
%that a deployment conveys no information by itself: the branch $\sigma^j_{\pi_t}$ is a function
%of the algorithm's own policy, and $p_j$ is revealed only through the samples.

For ease of notation, let us write $X := \gamma\eps/(1-\gamma)$ and $G_j(\mu) := \Gglob_j(\mu)$ the global stability gap on the $j$-th hidden action instance. We will need three definitions: (a) $q_j(\mu) := \mu\big(\{\pi : y^j_\pi < \tfrac12\}\big)$, the probability that $\pi$ underweights the hidden action $j$; (b) $\sigma^j_\pi := +1$ if $y^j_\pi < \tfrac12$ and $-1$ otherwise; and (c) $y^j_\pi := \pi(j\mid s_0)$, the probability of playing the hidden action $j$ under policy $\pi$.

By
Lemma~\ref{lem:hiddengap}, $G_j(\mu) = X\big[(2q_j(\mu)-1)_+ - \E_\mu[y^j_\pi\sigma^j_\pi]\big]$,
so $G_j(\mu)/X \in [-1,2]$ and in particular $|G_j| \le 2X$ for every $\mu$ and every $j$. For
the empirical mixture, $q_j(\hat\mu) = \tfrac1T\set{t : y^j_{\pi_t}<\tfrac12}$ and
$\E_{\hat\mu}[y^j\sigma^j] = \tfrac1T\sum_t y^j_{\pi_t}\sigma^j_{\pi_t}$ are measurable functions
of $(\pi_1,\dots,\pi_T)$, and $G_j(\hat\mu)$ is a bounded random variable. We will also write $\gD$ to denote the dataset of all the trajectories drawn by the algorithm, $\mathbb{P}_j$ to denote the law of $\gD$ under the $j$-th hidden action instance, and $\mathbb{P}_0$ to denote the law of $\gD$ under the null instance in which $p_a \equiv \tfrac12$ for every $a$.

\emph{The null instance.} Lemma~\ref{lem:counting} holds for \emph{every}
$\mu\in\Delta(\Pi)$, hence pathwise for the random $\hat\mu$. Taking expectation under
$\mathbb{P}_0$ and exchanging the finite sum with the expectation,
\begin{equation}\label{eq:step1}
  \frac1K\sum_{j=1}^{K}\E_{\gD \sim \mathbb{P}_0}\big[G_j(\hat\mu(\gD))\big] \;\ge\; X\Big(1-\frac5K\Big).
\end{equation}

\emph{Transferring to $\mathbb{P}_j$.} For a random variable $Z$ with $|Z|\le M$ and two
laws $P,Q$ on the same space, $|\E_P Z - \E_Q Z| \le 2M\|P-Q\|_{\mathrm{TV}}$. Applying this with
$Z = G_j(\hat\mu)$ and $M = 2X$,
\[
  \E_{\gD \sim \mathbb{P}_j}\big[G_j(\hat\mu(\gD))\big] \;\ge\; \E_{\gD \sim \mathbb{P}_0}\big[G_j(\hat\mu(\gD))\big] - 4X\,\big\|\mathbb{P}_j-\mathbb{P}_0\big\|_{\mathrm{TV}} .
\]

\emph{Bounding the total variation on average.} By Pinsker's inequality and
Lemma~\ref{lem:kl},
\[
  \|\mathbb{P}_j-\mathbb{P}_0\|_{\mathrm{TV}} \le \sqrt{\tfrac12\mathrm{KL}(\mathbb{P}_0\|\mathbb{P}_j)}
  \le \sqrt{\tfrac32\eps^{2}\,\E_0[N_j]}.
\]
Averaging over $j$ and using the concavity of the square
root (Jensen) together with $\sum_j \E_0[N_j] \le N$,
\begin{equation}\label{eq:step3}
  \frac1K\sum_{j=1}^{K}\big\|\mathbb{P}_j-\mathbb{P}_0\big\|_{\mathrm{TV}}
  \;\le\; \sqrt{\frac{3\eps^{2}}{2}\cdot\frac1K\sum_j \E_0[N_j]}
  \;\le\; \sqrt{\frac{3\,\eps^{2}\,N}{2K}} .
\end{equation}
%This averaging replaces the intersection of two index sets: no index needs to be both
%uninformative and bad, because both quantities are controlled in mean over $j$.

From \eqref{eq:step1} and \eqref{eq:step3},
\[
  \frac1K\sum_{j=1}^{K}\E_{\gD \sim \mathbb{P}_j}\big[G_j(\hat\mu(\gD))\big]
  \;\ge\; X\left(1 - \frac5K - 4\sqrt{\frac{3\eps^{2}N}{2K}}\right).
\]
Since $\eps = \eps_P/2$, the hypothesis $N \le K/(24\eps_P^2)$ reads $N \le K/(96\eps^2)$, which implies
$\tfrac{3\eps^2N}{2K} \le \tfrac{1}{64}$ and $4\sqrt{\frac{3\eps^{2}N}{2K}} \le \tfrac12$; and $K\ge20$ gives
$5/K \le \tfrac14$. This gives a lower bound of at least $\tfrac14$, and
$X/4 = \gamma\eps_P/(8(1-\gamma))$. A maximum dominates an average, giving the final claim.
\end{proof}

\begin{lemma}\label{lem:hiddengap}
Write $y^j_\pi := \pi(j \mid s_0)$, $\sigma^j_\pi := +1$ if $y^j_\pi < \tfrac12$ and $-1$
otherwise, and $q_j(\mu) := \mu\big(\{\pi : y^j_\pi < \tfrac12\}\big)$. Then for every
$\mu \in \Delta(\Pi)$,
\[
  \Gglob(\mu) \;=\; \frac{\gamma\,\eps}{1-\gamma}\Big[\big(2q_j(\mu)-1\big)_{+}
  \;-\; \E_{\pi\sim\mu}\big[y^j_\pi \sigma^j_\pi\big]\Big].
\]
\end{lemma}

\begin{proof}
For any $\pi'$ and any realised model $M$, $V_M(\pi') = \tfrac{\gamma}{1-\gamma}\sum_a \pi'(a\mid s_0)p_a$,
since reaching $\top$ at step $1$ yields the discounted unit stream. As $p_a = \tfrac12$ for
$a \ne j$, this equals 
$$
V_{M}(\pi') = \frac{\gamma}{1-\gamma}\big[(1-y^j_{\pi'}) \tfrac{1}{2} + p_j\cdot y^j_{\pi'}\big] = \frac{\gamma}{1-\gamma}\big[\tfrac12 + y^j_{\pi'}(\eps - \tfrac12)\big].
$$
Taking $\pi' = \pi$ and averaging gives
$\E_\mu[V_{M(\pi)}(\pi)] = \tfrac{\gamma}{1-\gamma}\big[\tfrac12 + \eps\,\E_\mu[y^j_\pi\sigma^j_\pi]\big]$.
For the comparator policy, $\E_\mu[p_j(\pi)] - \tfrac12 = \eps(2q_j(\mu)-1)$, so
$\E_\mu[V_{M(\pi)}(\pi')] = \tfrac{\gamma}{1-\gamma}\big[\tfrac12 + y^j_{\pi'}\eps(2q_j(\mu)-1)\big]$. This is
 is linear in $y^j_{\pi'} \in [0,1]$ and is maximum value is exactly 
$\tfrac12 + \eps(2q_j(\mu)-1)_+$. Subtracting gives the claim.
\end{proof}

\begin{lemma}[Counting]\label{lem:counting}
For every $\mu \in \Delta(\Pi)$,
\[
  \frac{1}{K}\sum_{j=1}^{K}\Gglob_{j}(\mu) \;\ge\; \frac{\gamma\,\eps}{1-\gamma}\Big(1 - \frac{5}{K}\Big),
\]
where $\Gglob_{j}$ is the global stability gap in the instance with hidden action $j$. %Moreover
%$\Gglob_{j}(\mu) \le 2\gamma\eps/(1-\gamma)$ for every $j$, so at least $K/8$ of the indices
%satisfy $\Gglob_{j}(\mu) \ge \gamma\eps/(4(1-\gamma))$ whenever $K \ge 10$.
\end{lemma}

\begin{proof}
  \Cref{lem:hiddengap} establishes the following bound 
  $$\Gglob_{j}(\mu) = \tfrac{\gamma\,\eps}{1-\gamma}\big[(2q_j-1)_+ - \E_\mu[y^j\sigma^j]\big].$$
Since $y\sigma \le y$ always, $\E_\mu[y^j_\pi\sigma^j_\pi] \le \E_\mu[y^j_\pi]$, and
$\tfrac1K\sum_j \E_\mu[y^j_\pi\sigma^j_\pi] = \E_\mu\big[\tfrac1K\sum_j \pi(j\mid s_0)\big] = \tfrac1K$. For
the first term, each policy can have $y^j_\pi \ge \tfrac12$ for at most two indices $j$, so
$$
\sum_j (1-q_j(\mu)) = \sum_j \mu\left( \set{\pi: y^j_\pi \ge \tfrac{1}{2}}\right) \le 2.
$$
This gives us $\tfrac1K\sum_j(2q_j-1) \ge 1 - \tfrac4K$. Combining with
$(x)_+ \ge x$ and Lemma~\ref{lem:hiddengap} gives the average bound. %The upper bound follows
%from $(2q_j-1)_+ \le 1$ and $y^j\sigma^j \ge -1$. 
%For the last claim, if a fraction $p$ of indices had gap at least $a := \tfrac{\gamma\eps}{4(1-\gamma)}$ then the average would be at
%most $2pX + a$ with $X := \gamma\eps/(1-\gamma)$; the average is at least $X/2$ for $K \ge 10$,
%forcing $p \ge 1/8$.
\end{proof}

\begin{lemma}[Divergence decomposition]\label{lem:kl}
Let $\mathbb{P}_j$ be the law of the algorithm's trajectory of observations under the instance of
Definition~\ref{def:hidden} with hidden action $j$, and $\mathbb{P}_0$ the law under the null
instance in which $p_a \equiv \tfrac12$ for every $a$. If $\eps \le 1/4$ then for every
$j\in[K]$,
\[
  \mathrm{KL}\big(\mathbb{P}_0\,\big\|\,\mathbb{P}_j\big) \;\le\; 3\eps^{2}\,\E_0\big[N_j\big].
  %\qquad\text{and}\qquad \sum_{j\in[K]}\E_0\big[N_j\big] \;\le\; N .
\]
\end{lemma}

\begin{proof}
  As in the proof of \Cref{thm:linear-lower-bound}, let us define $\sigma^j_{\pi_t} := +1$ if $\pi_t(j\mid s_0) < \tfrac12$ and $-1$ otherwise. Now, conditioned on the history preceding the draw $(t,i)$, the deployed policy $\pi_t$,  the
selected action $A_{t,i}$, and the branch indicator $\sigma^j_{\pi_t}$ are measurable. 

If $A_{t,i}\ne j$ the two models both return $\mathrm{Bern}(\tfrac12)$ and
the conditional divergence is $0$. If $A_{t,i}=j$ the conditional laws are
$\mathrm{Bern}(\tfrac12)$ under $\mathbb{P}_0$ and
$\mathrm{Bern}\big(\tfrac12+\eps\sigma^j_{\pi_t}\big)$ under $\mathbb{P}_j$, so the conditional
divergence is
\[
  \mathrm{kl}\Big(\tfrac12\,\Big\|\,\tfrac12\pm\eps\Big)
  = \tfrac12\ln\frac{1/4}{1/4-\eps^{2}} = -\tfrac12\ln\big(1-4\eps^{2}\big),
\]
The value is identical for both branches.  For
$\eps\le1/4$ we have $4\eps^2\le\tfrac14$ and $-\ln(1-x)\le \tfrac{x}{1-x}\le\tfrac43 x$ on
$[0,\tfrac14]$, so this is at most $\tfrac83\eps^{2}\le 3\eps^{2}$. The chain rule for
Kullback--Leibler divergence over the (finitely many) draws, together with the fact that the
algorithm's internal randomisation has the same law under both models, gives
$\mathrm{KL}(\mathbb{P}_0\|\mathbb{P}_j) \le 3\eps^2\,\E_0[N_j]$.
\end{proof}

\subsection{Multi-Agent Setting}
\begin{proof}[Proof of \Cref{thm:multi-agent-local-stability}]
Fix a player $i$, a comparator $\pi' \in \Pi^i$ and a state $s$. Player $i$'s realised gains are
$g^i_t(s,\cdot) = w_t(s)Q^i_t(s,\cdot) \in [0, w_t(s)/(1-\gamma)]^{\gA_i}$, and \eqref{eq:algG} is
performing Hedge on them, so Lemma~\ref{lem:hedge} applies verbatim. Summing over $s$ and using
$\sum_s w_t(s)^2 \le 1$ gives us
\begin{align*}
  \sum_{t=1}^{T}\sum_{s} w_t(s)\,\big\langle \pi'(\cdot\mid s)-\pi^i_t(\cdot\mid s),\ Q^i_t(s,\cdot)\big\rangle
  \;&\le\; \frac{S\ln A_i}{\eta_i} + \frac{\eta_i }{8(1-\gamma)^2} \sum_{t=1}^T \sum_s w_t(s)^2\\
  &\le \frac{S\ln A_i}{\eta_i} + \frac{\eta_i T}{8(1-\gamma)^2}.
\end{align*}
Dividing by $T(1-\gamma)$, and substituting $\eta_i = 2(1-\gamma)\sqrt{2S\ln A_i/T}$ and maximising over $\pi'$ gives us the
per-player upper bound. Note that the per-player bound is deterministic, and taking the maximum over the $n$ players gives us the upper bound on $\Gloc(\hat\mu)$.
\end{proof}

\begin{proof}[Proof of \Cref{prop:multisep}]
\emph{The two components share their occupancies.} Player $1$'s per-step survival probability is
$\langle\bar\pi^1,\ell^{A}\rangle = \langle\bar\pi^1,\ell^{B}\rangle = 1-\eps$, so the state
occupancy $d$ and player $1$'s value function $V$ along the chain are identical under the two policies $\pi^2_A$ and $\pi^2_B$ of player $2$. We will write $d(s_h)$ to denote the state occupancy and $V^1(s_{h+1})$ for the common value function of player $1$.

\emph{Player $1$ is locally indifferent.} By Proposition~\ref{prop:ccemulti},
$\bar q^{\,1}_{\mu^\star}(s_h,j) = d(s_h)\,\gamma V^1(s_{h+1})\cdot\tfrac12\big(\ell^{A}(j)+\ell^{B}(j)\big)$,
and $\ell^{A}(1)+\ell^{B}(1) = \ell^{A}(2)+\ell^{B}(2) = 2-2\eps$, so this is constant in $j$ and
equals $d(s_h)\gamma V^1(s_{h+1})(1-\eps)$. The deployed term $\bar v^{\,1}_{\mu^\star}(s_h)$ can be computed as
\begin{align*}
\bar v^{\,1}_{\mu^\star}(s_h) &= d(s_h) \tfrac{1}{2}\left( Q^1(s_h,1) + Q^1(s_h,2) \right)\\
&= d(s_h)\,\gamma V^1(s_{h+1})(1-\epsilon).
\end{align*}
Therefore, $\bar v^{\,1}_{\mu^\star}(s_h)$ takes
the same value as $\bar q^{\,1}_{\mu^\star}(s_h,j)$, so every per-state gap vanishes and according to \Cref{prop:ccemulti}, $\Glocm{1}(\mu^\star)=0$. 
%Averaging over the two components leaves the comparator with no preference between actions --- the same cancellation %as in Theorem~\ref{thm:D}.

\emph{Player $2$ is locally indifferent.} Against $\bar\pi^1$, player $2$'s choice gives the same survival probability $1-\eps$ to player $1$, so it gives the same probability $\eps$ of reaching state $\bot$ for player $2$. Therefore, by a very similar calculation to the one above, we can show that $\bar{q}^2_{\mu^\star}(s_h,b)$ is constant in $b$ and equals $\bar{v}^2_{\mu^\star}(s_h)$. Therefore, 
$\Glocm{2}(\mu^\star)=0$, and $\Gloc(\mu^\star)= \max_{i\in\{1,2\}} \Glocm{i}(\mu^\star) = 0$.

\emph{Player $1$ gains globally.} Let us write $\sigma$ to denote a comparator policy of player $1$.With $F_{A}(y) := \prod_h\big(1-2\eps(1-y_h)\big)$ and
$F_{B}(y) := \prod_h\big(1-2\eps y_h\big)$ as in \eqref{eq:FAFB},
$\E_{\mu^\star}[V^1(\sigma,\pi^2)] = \kappa\cdot\tfrac12\big(F_A(\sigma)+F_B(\sigma)\big)$, where $\kappa = \gamma^H/(1-\gamma)$. In particular, when $\sigma = \pi^1_1$ (i.e. always play $1$), we have $\E_{\mu^\star}[V^1(\sigma,\pi^2)] = \kappa\cdot\tfrac12\big(1+(1-2\eps)^{H}\big)$.  The deployed value i.e. $\E_{\mu^\star}[V^1(\bpi)]$ equals
$\kappa(1-\eps)^{H}$, as player $1$'s probability of survival is $1-\eps$ at every step. Therefore,  $\Gglobm{1}(\mu^\star) \ge \kappa f(\eps,H)$, where $f(\eps,H) = \tfrac{1}{2}(1+(1-2\eps)^{H}) - (1-\eps)^H$.  The final bound now follows by using the lower bound on $f(\eps,H)$ from \Cref{lem:lbd-feps}. 
\end{proof}

\begin{proposition}[State-wise CCE decomposition]\label{prop:ccemulti}
With
\[
  \bar q^{\,i}_\mu(s,a) := \E_\mu\big[d^{\bpi}_{P_{\bpi}}(s)Q^{i,\bpi}_{M(\bpi)}(s,a)\big],
  \qquad
  \bar v^{\,i}_\mu(s) := \E_\mu\big[d^{\bpi}_{P_{\bpi}}(s)\langle \pi^i(\cdot\mid s), Q^{i,\bpi}_{M(\bpi)}(s,\cdot)\rangle\big],
\]
\[
 \Glocm{i}(\mu) \;=\; \frac{1}{1-\gamma}\sum_{s\in\gS}\Big(\max_{a\in\gA_i}\bar q^{\,i}_\mu(s,a) - \bar v^{\,i}_\mu(s)\Big).
\]
\end{proposition}

\begin{proof}
Identical to Proposition~\ref{prop:cce}: the objective depends on $\pi'$ only through
$\pi'(\cdot\mid s)$, the constraint set $\Delta(\gA_i)^{\gS}$ is a product over $s$, and the
inner maximum of a linear functional over $\Delta(\gA_i)$ is attained at a vertex.
\end{proof}

\subsection{Performative Markov Potential Games}

\begin{proof}[Proof of \Cref{thm:convergence-potential-mpg}]
Recall the definition of global stability gap for an agent $i$.
$$
\Gglob_i(\widehat{\mu}) = \max_{\sigma \in \bm \Pi_i} \E_{\bpi \sim \widehat{\mu}} \left[ V^i_{M(\bpi)}(\sigma, \bpi^{-i}) - V^i_{M(\bpi)}(\pi^i,\bpi^{-i})\right]
$$
We first bound $\E_{\bpi \sim \widehat{\mu}} \left[ V^i_{M(\bpi)}(\sigma, \bpi^{-i}) - V^i_{M(\bpi)}(\pi^i,\bpi^{-i})\right]$ for a fixed policy profile $\bpi$. Holding $\bpi^{-i}$ fixed, player $i$ faces a stationary single-agent MDP on $\gS$ with action set
$\gA_i$, transition kernel $\E_{a^{-i}\sim\bpi^{-i}(\cdot| s)}[P(\cdot| s,(a^i,a^{-i}))]$ and
the correspondingly marginalised action-value function is $Q^{i,\bpi}_M$. Applying
Lemma~\ref{lem:pdl} in that MDP, we obtain
\begin{align*}
  V^i_{M(\bpi)}(\sigma,\bpi^{-i}) - V^i_{M(\bpi)}(\bpi)
  &= \frac{1}{1-\gamma}\sum_s d^{(\sigma,\bpi^{-i})}_{M(\bpi)}(s)\,
    \big\langle \sigma(\cdot\mid s)-\pi^i(\cdot\mid s),\ Q^{i,\bpi}_{M(\bpi)}(s,\cdot)\big\rangle\\
  &= \frac{1}{1-\gamma}\sum_s \frac{d^{(\sigma,\bpi^{-i})}_{M(\bpi)}(s)}{d^{\bpi}_{M(\bpi)}(s)}\,d^{\bpi}_{M(\bpi)}(s)\, \max_{\bar{\sigma} \in \Delta(\gA) } \big\langle \bar{\sigma}-\pi^i(\cdot\mid s),\ Q^{i,\bpi}_{M(\bpi)}(s,\cdot)\big\rangle.
\end{align*}
Now observe that the term $\max_{\bar{\sigma} \in \Delta(\gA) } \big\langle \bar{\sigma}-\pi^i(\cdot\mid s),\ Q^{i,\bpi}_M(s,\cdot)\big\rangle$ is non-negative as one can just take $\bar{\sigma}$ to be the action that maximizes $Q^{i,\bpi}_{M(\bpi)}(s,\cdot)$. Therefore,
\begin{align*}
  V^i_{M(\bpi)}(\sigma,\bpi^{-i}) - V^i_{M(\bpi)}(\bpi) &\le \norm{\frac{d^{(\sigma,\bpi^{-i})}_{M(\bpi)}}{d^{\bpi}_{M(\bpi)}}}_\infty \sum_s \frac{d^{\bpi}_{M(\bpi)}(s)}{1-\gamma}\, \max_{\bar{\sigma} \in \Delta(\gA) } \big\langle \bar{\sigma}-\pi^i(\cdot\mid s),\ Q^{i,\bpi}_{M(\bpi)}(s,\cdot)\big\rangle\\
  &\le \norm{\frac{d^{(\sigma,\bpi^{-i})}_{M(\bpi)}}{d^{\bpi}_{M(\bpi)}}}_\infty \max_{\bar{\sigma} \in \Delta(\gA)^\gS} \sum_s \frac{d^{\bpi}_{M(\bpi)}(s)}{1-\gamma}\,\big\langle \bar{\sigma}-\pi^i(\cdot\mid s),\ Q^{i,\bpi}_{M(\bpi)}(s,\cdot)\big\rangle\\
  &= \norm{\frac{d^{(\sigma,\bpi^{-i})}_{M(\bpi)}}{d^{\bpi}_{M(\bpi)}}}_\infty \max_{\bar{\sigma} \in \Delta(\gA)^\gS} \left \langle \bar{\sigma} - \pi^i, \nabla_{\bar{\pi}^i} V^i_{M(\bpi)}(\bar{\pi}^i,\bpi^{-i})\big |_{\bar{\pi}^i = \pi^i}\right \rangle
\end{align*}
The last line uses the policy gradient theorem~\citep{SMSM99} for a single agent, i.e.
$$\frac{\partial V_M(\pi)}{\partial \pi(a|s)} = d^\pi_M(s) Q^\pi_M(s,a)/(1-\gamma).$$
Instantiating the above inequality for each $\bpi=\bpi_t$ and averaging over $T$ iterations we obtain,
\begin{align*}
  &\frac{1}{T} \sum_{t=1}^T V^i_{M(\bpi_t)}(\sigma,\bpi^{-i}_t) - V^i_{M(\bpi_t)}(\bpi_t) \\
  \le &\max_{t\in [T]}\norm{\frac{d^{(\sigma,\bpi^{-i}_t)}_{M(\bpi_t)}}{d^{\bpi_t}_{M(\bpi_t)}}}_\infty \frac{1}{T} \sum_{t=1}^T \underbrace{ \max_{\bar{\sigma} \in \Delta(\gA)^\gS} \left \langle \bar{\sigma} - \pi^i_t, \nabla_{\bar{\pi}^i} V^i_{M(\bpi_t)}(\bar{\pi}^i,\bpi^{-i}_t)\big |_{\bar{\pi}^i = \pi^i_t}\right \rangle}_{:= \mathrm{FW}_i(\bpi_t)}.
\end{align*}
 Taking maximum over $\sigma \in \bm \Pi^i$ and using the definition of $\Gglob_i(\widehat{\mu})$ we obtain the following upper bound.
\begin{align*}
  \Gglob_i(\widehat{\mu}) \le &\max_{\sigma \in \bm \Pi^i} \max_{t\in [T]}\norm{\frac{d^{(\sigma,\bpi^{-i}_t)}_{M(\bpi_t)}}{d^{\bpi_t}_{M(\bpi_t)}}}_\infty \frac{1}{T} \sum_{t=1}^T \mathrm{FW}_i(\bpi_t).
\end{align*}
We now bound the Frank-Wolfe gap $\mathrm{FW}_i(\bpi_t)$. Let us define $g^i = \nabla_{\pi^i} \Phi_t(\bpi_t)$. Then for any $z \in \bm \Pi_i$, the variational inequality gives
$$
\left \langle \pi_t^i + \alpha g^i - \pi_{t+1}^i, z - \pi_{t+1}^i \right \rangle \le 0,
$$
and after rearranging we obtain,
$$
\alpha \left \langle g^i, z - \pi^i_{t+1}\right \rangle \le \left \langle \pi_{t+1}^i  - \pi_{t}^i, z - \pi_{t+1}^i \right \rangle \le \norm{\pi^i_{t+1} - \pi^i_t}_2 \norm{z - \pi^i_{t+1}}_2 \le \sqrt{2S} \norm{\pi^i_{t+1} - \pi^i_t}_2.
$$
The last inequality uses the observation
$$\norm{z-\pi^i_{t+1}}_2 = \sqrt{\sum_{s,a} (z(a|s) - \pi^i_{t+1}(a|s))^2} \le \sqrt{\sum_{s,a} z(a|s)^2 + \pi^i_{t+1}(a|s)^2} \le \sqrt{2S}.$$
Therefore,
$$
\left \langle g^i, z - \pi^i_t\right \rangle = \left \langle g^i, z - \pi^i_{t+1} \right \rangle + \left \langle g^i, \pi^i_{t+1} - \pi^i_{t} \right \rangle \le \left( G + \frac{\sqrt{2S}}{\alpha} \right) \norm{\pi^i_{t+1} - \pi^i_t}_2
$$
by \Cref{ass:mpgreg}. Taking maximum over $z \in \bm \Pi^i$ and averaging over $T$ iterations we obtain,
$$
\frac{1}{T} \sum_{t=1}^T  \mathrm{FW}_i(\bpi_t) \le \left( G + \frac{\sqrt{2S}}{\alpha} \right) \frac{1}{T} \sum_{t=1}^T \norm{\pi^i_{t+1} - \pi^i_t}_2.
$$
This gives us
\begin{align*}
  \Gglob(\widehat{\mu}) &\le \max_i \Gglob_i(\widehat{\mu}) \le \max_{i,\sigma\in \bm\Pi^i, t \in [T]} \norm{\frac{d^{(\sigma,\bpi^{-i}_t)}_{M(\bpi_t)}}{d^{\bpi_t}_{M(\bpi_t)}}}_\infty \left( G + \frac{\sqrt{2S}}{\alpha} \right) \frac{1}{T} \sum_{t=1}^T \sum_{i=1}^n\norm{\pi^i_{t+1} - \pi^i_t}_2\\
  &\le  \max_{i,\sigma\in \bm\Pi^i, t \in [T]} \norm{\frac{d^{(\sigma,\bpi^{-i}_t)}_{M(\bpi_t)}}{d^{\bpi_t}_{M(\bpi_t)}}}_\infty \left( G + \frac{\sqrt{2S}}{\alpha} \right) \frac{\sqrt{n}}{T} \sum_{t=1}^T \norm{\bpi_{t+1} - \bpi_t}_2\\
  &\le \max_{i,\sigma\in \bm\Pi^i, t \in [T]} \norm{\frac{d^{(\sigma,\bpi^{-i}_t)}_{M(\bpi_t)}}{d^{\bpi_t}_{M(\bpi_t)}}}_\infty \left( G + \frac{\sqrt{2S}}{\alpha} \right) \sqrt{n} \left( 2\alpha \eps_L + \sqrt{\frac{2\alpha}{T(1-\gamma)}}\right).
\end{align*}
For the second result, note that choosing $T \ge \frac{1}{(1-\gamma)\alpha \eps_L^2}$ gives us the following upper bound.
\begin{align*}
\Gglob(\widehat{\mu}) &\le 4\alpha \sqrt{n} \eps_L \left( G + \sqrt{2S}/\alpha \right) \max_{i,\sigma\in \bm\Pi^i, t \in [T]} \norm{\frac{d^{(\sigma,\bpi^{-i}_t)}_{M(\bpi_t)}}{d^{\bpi_t}_{M(\bpi_t)}}}_\infty\\
&\le 4\eps_L\sqrt{n}(G\alpha + \sqrt{2S})\max_{i,\sigma\in \bm\Pi^i, t \in [T]} \norm{\frac{d^{(\sigma,\bpi^{-i}_t)}_{M(\bpi_t)}}{d^{\bpi_t}_{M(\bpi_t)}}}_\infty\\
&\le 8\sqrt{2Sn} \eps_L \max_{i,\sigma\in \bm\Pi^i, t \in [T]} \norm{\frac{d^{(\sigma,\bpi^{-i}_t)}_{M(\bpi_t)}}{d^{\bpi_t}_{M(\bpi_t)}}}_\infty
\end{align*}
as long as $\alpha \le \sqrt{2S}/G$.
\end{proof}

\begin{lemma}\label{lem:onestep}
Let $\Delta_t := \|\bpi_{t+1}-\bpi_t\|_2$ and . Under \Cref{ass:mpgreg} and $\alpha \le 1/L_\Phi$,
$$
\frac{1}{T} \sum_{t=1}^T \Delta_t \le 2\alpha \eps_L + \sqrt{\frac{2\alpha}{T(1-\gamma)}}.
$$
\end{lemma}

\begin{proof}
  We will write $\psi(\bpi) := \Phi_{M(\bpi)}(\bpi)$ and 
 $g := \nabla\Phi_t(\bpi_t)$. The projection $\bpi_{t+1}$ of $\bpi_t + \alpha g$ onto $\bm\Pi$ satisfies the
variational inequality $\langle \bpi_t + \alpha g - \bpi_{t+1},\, \bm z - \bpi_{t+1}\rangle \le 0$
for every $\bm z \in \bm\Pi$. Taking $\bm z = \bpi_t$ and rearranging,
\begin{equation}\label{eq:varineq}
  \Delta_t^{2} \;\le\; \alpha\,\langle g,\ \bpi_{t+1}-\bpi_t\rangle .
\end{equation}
Since $\Phi_t(\cdot)$ is $L_\Phi$-smooth.
\begin{align*}
  \Phi_t(\bpi_{t+1}) &\ge\; \Phi_t(\bpi_t) + \langle g,\ \bpi_{t+1}-\bpi_t\rangle - \frac{L_\Phi}{2}\Delta_t^{2}\\
  &\ge\; \Phi_t(\bpi_t) + \frac{\Delta_t^{2}}{\alpha} - \frac{L_\Phi}{2}\Delta_t^{2}\\
  &\ge\; \Phi_t(\bpi_t) + \frac{\Delta_t^{2}}{2\alpha},
\end{align*}
using \eqref{eq:varineq} and then $\alpha \le 1/L_\Phi$. Finally, we obtain the following inequality.
\begin{align*}
  \psi(\bpi_{t+1}) - \psi(\bpi_t)
  &= \big[\Phi_{t+1}(\bpi_{t+1}) - \Phi_t(\bpi_{t+1})\big]
  + \big[\Phi_t(\bpi_{t+1}) - \Phi_t(\bpi_t)\big]\\
  &\ge \Phi_{M(\bpi_{t+1})}(\bpi_{t+1}) - \Phi_{M(\bpi_{t})}(\bpi_{t+1}) + \frac{\Delta_t^2}{2\alpha}\\
  &\ge -\eps_L \norm{\bpi_{t+1} - \bpi_t}_2 + \frac{\Delta_t^2}{2\alpha}.
\end{align*}
The last line uses (P2) of \Cref{ass:mpgreg}. This gives us the following recurrence relation.
\begin{align}
  \label{eq:recurrence-relation}
  \psi(\bpi_{t+1}) - \psi(\bpi_t) \ge -\eps_L \Delta_t + \frac{\Delta_t^2}{2\alpha}
\end{align}
Summing \Cref{eq:recurrence-relation} for $t=1,\ldots,T$, we obtain
\begin{align*}
  \psi(\bpi_{T+1}) - \psi(\bpi_1) \ge -\eps_L \sum_{t=1}^T \Delta_t + \frac{1}{2\alpha} \sum_{t=1}^T \Delta_t^2.
\end{align*}
We now use two observations. First, rewards are bounded between $0$ and $1$ and
$$\psi(\bpi_{T+1}) - \psi(\bpi_1) = \Phi_{M(\bpi_{T+1})}(\bpi_{T+1}) - \Phi_{M(\bpi_1)}(\bpi_1) \le 1/(1-\gamma).$$
Second, Cauchy-Schwarz inequality gives $\sum_{t=1}^T \Delta_t \le \sqrt{T\ \sum_{t=1}^T \Delta_t^2 }$. Writing $S_2^2 = \sum_{t=1}^T \Delta_t^2$ we obtain,
$$
\frac{1}{1-\gamma} \ge -\eps_L \sqrt{T}\cdot S_2 \ +\ \frac{1}{2\alpha} S_2^2.
$$
This gives us $S_2 \le \alpha \left( \eps_L \sqrt{T} + \sqrt{\eps_L^2 T + \frac{2}{\alpha(1-\gamma)}} \right) \le 2\alpha \eps_L \sqrt{T} + \sqrt{\frac{2\alpha}{1-\gamma}}$. Now we use $\frac{1}{T} \sum_{t=1}^T \Delta_t \le \frac{S_2}{\sqrt{T}}$ to obtain the desired bound.
\end{proof}

\end{document}

%% file: math_commands.tex
\usepackage{amsmath,amsfonts,bm}

\def\eqref#1{equation~\ref{#1}}
\def\1{\bm{1}}

\def\eps{{\epsilon}}

\DeclareMathAlphabet{\mathsfit}{\encodingdefault}{\sfdefault}{m}{sl}
\SetMathAlphabet{\mathsfit}{bold}{\encodingdefault}{\sfdefault}{bx}{n}

\def\gA{{\mathcal{A}}}

\def\gD{{\mathcal{D}}}

\def\gF{{\mathcal{F}}}

\def\gS{{\mathcal{S}}}

\newcommand{\E}{\mathbb{E}}

\newcommand{\R}{\mathbb{R}}

\newcommand{\Var}{\mathrm{Var}}

\DeclareMathOperator*{\argmax}{arg\,max}